\documentclass{article}
\usepackage[utf8]{inputenc}
\usepackage[pdftex]{graphicx}
\usepackage[nonatbib,final]{neurips_2021}
\usepackage[numbers]{natbib} 
\usepackage{emile}
\usepackage{magic_dice}
\usepackage{listings}

\newtheorem{lemma}{Lemma}[theorem]

\theoremstyle{theorem}
\newtheorem{manualtheoreminner}{Theorem}
\newenvironment{manualtheorem}[1]{%
  \renewcommand\themanualtheoreminner{#1}%
  \manualtheoreminner
}{\endmanualtheoreminner}

\theoremstyle{theorem}
\newtheorem{manualpropinner}{Proposition}
\newenvironment{manualproposition}[1]{%
  \renewcommand\themanualpropinner{#1}%
  \manualpropinner
}{\endmanualpropinner}
\newcommand{\scg}{{\mathcal{G}}}
\newcommand{\nodes}{{\mathcal{N}}}
\newcommand{\edges}{{\mathcal{E}}}
\newcommand{\stochastic}{{\mathcal{S}}}
\newcommand{\deterministic}{{\mathcal{F}}}
\newcommand{\parameters}{{\Theta}}
\newcommand{\costs}{\mathcal{C}}
\newcommand{\stochnode}{S}
\newcommand{\detnode}{F}
\newcommand{\cost}{C}
\newcommand{\func}{f}
\newcommand{\node}{N}
\newcommand{\nodealt}{M}
\newcommand{\pa}{\mathrm{pa}}

\newcommand{\sspace}{\Omega}

\newcommand{\influences}{\prec}
\newcommand{\before}{\influences}

\newcommand{\sampleset}{\mathcal{X}}
\newcommand{\tproposal}{proposal distribution}
\newcommand{\tmultipl}{gradient function}
\newcommand{\tadditive}{control variate}
\newcommand{\magicbox}{\textsf{MagicBox}}

\newcommand{\amt}{{n_\stochnode}}
\newcommand{\costresult}{\cost}

\newcommand{\vals}[1]{{\bx_{{#1}}}}

\newcommand{\sample}[1]{\sampleset_{{#1}}}

\newcommand{\weight}[1]{{w_{#1}}}
\newcommand{\Weight}[1]{{W_{#1}}}
\newcommand{\multipl}[1]{{l_{#1}}}
\newcommand{\Multipl}[1]{{L_{#1}}}
\newcommand{\additive}[1]{{a_{#1}}}
\newcommand{\Additive}[1]{{A_{#1}}}
\newcommand{\proposal}[1]{q_{#1}}
\newcommand{\proposalcond}[1]{{q_{#1}}}
\newcommand{\baseline}[1]{{b_{#1}}}
\newcommand{\SL}{SL_{\operatorname{Storch}}}

\newcommand{\fweight}[1]{\weight{#1}(\vals{#1})}
\newcommand{\fmultipl}[1]{\multipl{#1}(\vals{#1})}
\newcommand{\fadditive}[1]{\additive{#1}(\vals{{\ifthenelse{\equal{#1}{1}}{#1}{< #1}}}, \sampleset_i)}
\newcommand{\fproposal}[1]{q(\sample{#1})}
\newcommand{\fproposalcond}[1]{q(\sample{#1}|\vals{\ifthenelse{\equal{#1}{2}}{1}{< #1}})}
\newcommand{\fbaseline}[1]{b_{#1}(\vals{{\ifthenelse{\equal{#1}{1}}{#1}{< #1}}}, \sampleset_i\setminus \{x_{#1}\})}

\newcommand{\itersample}[1]{\vals{#1}\in \sample{#1}}

\newcommand{\forward}{\overrightarrow}
\newcommand{\equivforward}{\forward{\equiv}}

\newcommand{\g}{g}

\newcommand{\gradestim}{\langle q_i, w_i, l_i, a_i \rangle}

\definecolor{codegreen}{rgb}{0,0.6,0}
\definecolor{codegray}{rgb}{0.5,0.5,0.5}
\definecolor{codepurple}{rgb}{0.58,0,0.82}
\definecolor{backcolour}{rgb}{0.95,0.95,0.92}

\lstdefinestyle{mystyle}{
    backgroundcolor=\color{backcolour},   
    commentstyle=\color{codegreen},
    keywordstyle=\color{magenta},
    numberstyle=\tiny\color{codegray},
    stringstyle=\color{codepurple},
    basicstyle=\ttfamily\footnotesize,
    breakatwhitespace=false,         
    breaklines=true,                 
    captionpos=b,                    
    keepspaces=true,                 
    numbers=left,                    
    numbersep=5pt,                  
    showspaces=false,                
    showstringspaces=false,
    showtabs=false,                  
    tabsize=2
}

\title{\emph{Storchastic}: A Framework for \\ General Stochastic Automatic Differentiation}
\author{Emile van Krieken \\
Vrije Universiteit Amsterdam \\
\texttt{e.van.krieken@vu.nl}
\And Jakub M. Tomczak \\ Vrije Universiteit Amsterdam \\ \texttt{j.m.tomczak@vu.nl} 
\And Annette ten Teije \\ Vrije Universiteit Amsterdam \\ \texttt{annette.ten.teije@vu.nl}}
\date{April 2020}

\begin{document}

\maketitle

    % However, current methods for automatic gradient estimation when modelers use sampling to estimate the intractable expectations common in Reinforcement Learning and Variational Inference are limited: 
    % However, current methods for automatic Monte Carlo gradient estimation are limited.
    % Modelers use gradient estimation when intractable expectations appear in their Deep Learning models, which are typically resolved using sampling.
    % The gradient estimation is used to 
    % This process is currently challenging when modelers use intractable expectations common in Reinforcement Learning and Variational Inference: They will have to resort to sampling, which is not differentiable.
    % When the modelers use sampling methods to estimate intractable expectations such as in Reinforcement Learning or Variational Inference, 
    % They often use sampling methods to estimate intractable expectations such as in Reinforcement Learning and Variational Inference.
    % Currently, automatic differentiation libraries are limited in estimating gradients through these sampling steps.
\begin{abstract}%
    Modelers use automatic differentiation (AD) of computation graphs to implement complex deep learning models without defining gradient computations.
    Stochastic AD extends AD to stochastic computation graphs with sampling steps, which arise when modelers handle the intractable expectations common in reinforcement learning and variational inference.
    However, current methods for stochastic AD are limited: 
    They are either only applicable to continuous random variables and differentiable functions, or can only use simple but high variance score-function estimators.
    To overcome these limitations, we introduce \emph{Storchastic}, a new framework for AD of stochastic computation graphs.
    \emph{Storchastic} allows the modeler to choose from a wide variety of gradient estimation methods at each sampling step, to optimally reduce the variance of the gradient estimates.
    Furthermore, \emph{Storchastic} is provably unbiased for estimation of any-order gradients, and generalizes variance reduction techniques to any-order derivative estimates. 
    Finally, we implement \emph{Storchastic} as a PyTorch library  at \url{github.com/HEmile/storchastic}.
    % \begin{itemize}
    %     \item This allows easy mix-matching and insightful comparisons of estimators.
    %     \item Framework makes process of deriving or extending estimators easier by disentangling their properties. 
    % \end{itemize}
\end{abstract}

\section{Introduction}
    One of the driving forces behind deep learning is automatic differentiation (AD) libraries of complex computation graphs. 
    Deep learning modelers are relieved by accessible AD of the need to implement  complex derivation expressions of the computation graph. 
    However, modelers are currently limited in settings where the modeler uses intractable expectations over random variables \cite{mohamedMonteCarloGradient2020, correiaEfficientMarginalizationDiscrete2020}. 
    Two common examples are reinforcement learning methods using policy gradient optimization \cite{williamsSimpleStatisticalGradientfollowing1992,lillicrapContinuousControlDeep2016, mnihAsynchronousMethodsDeep2016} % Add references here
    and latent variable models, especially when inferred using amortized variational inference \cite{mnihNeuralVariationalInference2014, kingmaAutoencodingVariationalBayes2014, ranganathBlackBoxVariational2014,rezendeStochasticBackpropagationApproximate2014}.
    % Latent variables are powerful since they allow the modeller to inject background knowledge into the generative process of the data. % This sentence might not be necessary
    Typically, modelers estimate these expectations using Monte Carlo methods, that is, sampling,
    % In this context, the corresponding random variables are called \emph{stochastic nodes}  of a \emph{stochastic} computation graph \cite{schulmanGradientEstimationUsing2015}.
    and resort to gradient estimation techniques \cite{mohamedMonteCarloGradient2020} to differentiate through the expectation.
    % These techniques often have high gradient variance, are often only derived for first-order gradient estimation, and can introduce bias when not properly implemented.
    % Therefore, we study how to implement low-variance gradient estimation techniques in AD libraries.
    % This requires an implementation that generalizes to higher-order gradient estimation \cite{anilScalableSecondOrder2021} while preventing the modeler from making implementation mistakes.

    % The current approaches that combine AD with gradient estimation can roughly be divided in two.
    A popular approach for stochastic AD is reparameterization~\cite{kingmaAutoencodingVariationalBayes2014}, which is both unbiased and has low variance, but is limited to continuous random variables and differentiable functions. 
    The other popular approach~\cite{williamsSimpleStatisticalGradientfollowing1992,schulmanGradientEstimationUsing2015,foersterDiCEInfinitelyDifferentiable2018} analyzes the computation graph and then uses the score function estimator to create a \emph{surrogate loss} that provides gradient estimates when differentiated. 
    While this approach is more general as it can also be applied to discrete random variables and non-differentiable functions, naive applications of the score function will have high variance, which leads to unstable and slow convergence.
    Furthermore, this approach is often implemented incorrectly \cite{foersterDiCEInfinitelyDifferentiable2018}, which can introduce bias in gradients.
    % This is a challenge for modelers that want to use expectations in their deep learning model, as they will still need to manually implement each lower-variance gradient estimation method that the modeler wants to compare the performance of. 
    
    \begin{figure}
        \includegraphics[width=\linewidth]{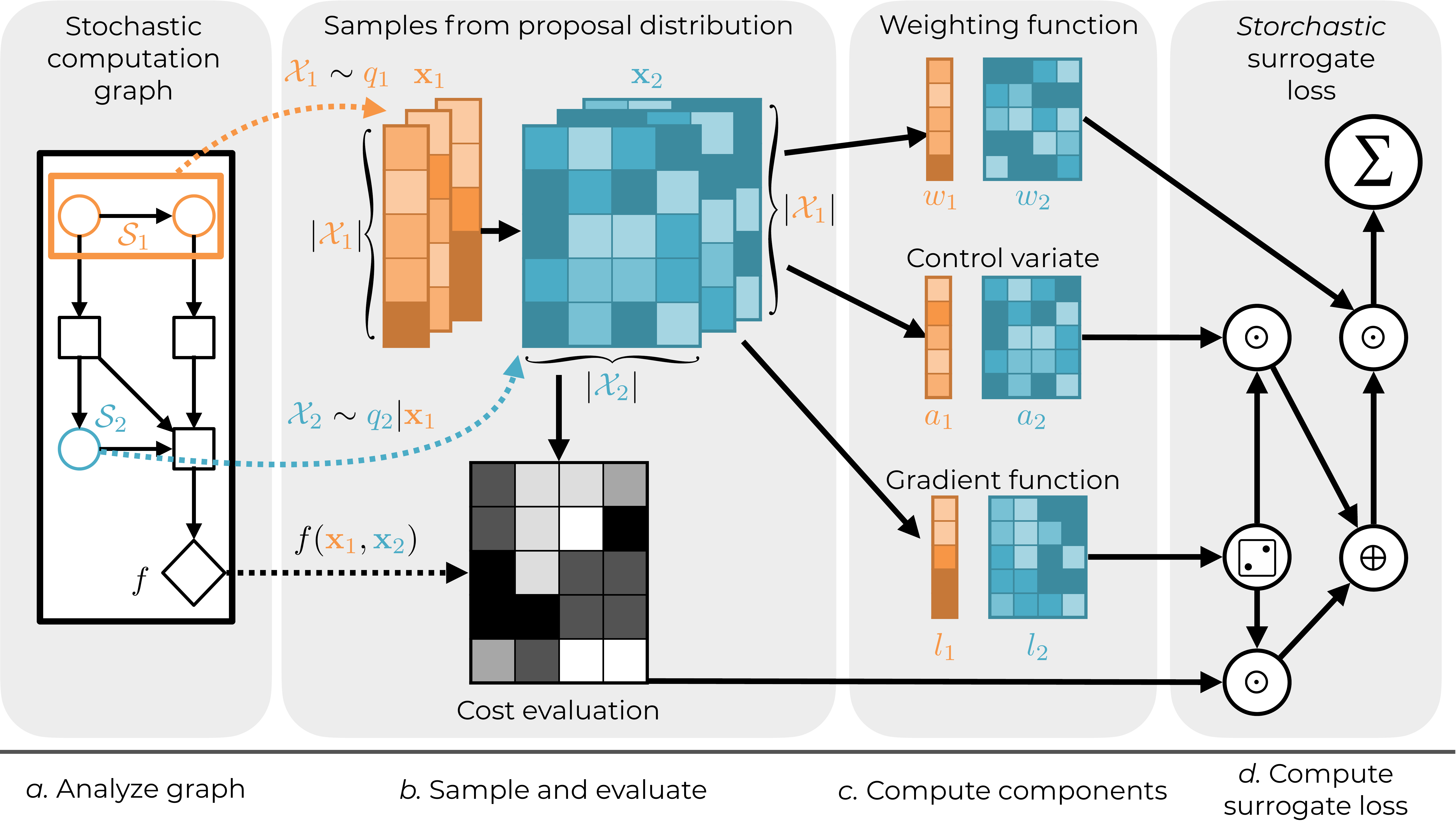}
        \caption{An illustration of the (parallelized) \emph{Storchastic} loss computation. 
        \emph{a.} Assign the stochastic nodes of the input stochastic computation graph  (SCG) into two topologically sorted partitions. 
        \emph{b.} Evaluate the SCG. We first sample the set of values $\mathcal{X}_1$ from the \tproposal{}. For each of the samples $\bx_i\in\mathcal{X}_1$, we then sample a set of samples $\mathcal{X}_2$. The rows in the figure indicate different samples in $\mathcal{X}_1$, while the columns indicate samples in $\mathcal{X}_2$. The different samples are used to evaluate the cost function $f$ $|\mathcal{X}_1|\cdot |\mathcal{X}_2|$ times. 
        \emph{c.} Compute the weighting function, \tadditive{} and \tmultipl{} for all samples. 
        \emph{d.} Using those components and the cost function evaluation, compute the \emph{storchastic} surrogate loss, mimicking Algorithm \ref{alg:storchastic}. $\odot$ refers to element-wise multiplication, $\oplus$ to element-wise summation and $\sum$ for summing the entries of a matrix. 
        % It uses broadcasting wherever necessary. 
        }
        \label{fig:surrogate-loss}
    \end{figure}

    We therefore develop a new framework called \emph{Storchastic} to support deep learning modelers. 
    They can use \emph{Storchastic} to focus on defining stochastic deep learning models without having to worry about complex gradient estimation implementations.
    \emph{Storchastic} extends DiCE \cite{foersterDiCEInfinitelyDifferentiable2018} to other gradient estimation techniques than basic applications of the score function. 
    % To this end, we introduce the \emph{Storchastic} framework.
    % Storchastic extends the DiCE estimator \cite{foersterDiCEInfinitelyDifferentiable2018} to other gradient estimation methods than the score function.
    It defines a surrogate loss by decomposing gradient estimation methods into four components: The \tproposal{}, weighting function, \tmultipl{} and \tadditive{}.  
    We can use this decomposition to get insight into how gradient estimators differ, and use them to further reduce variance by adapting components of different gradient estimators.

    Our main contribution is a framework with a formalization and a proof that, if the components satisfy certain conditions, performing $n$-th order differentiation on the \emph{Storchastic} surrogate loss gives unbiased estimates of the $n$-th order derivative of the stochastic computation graph.
    % To formalize this loss for the proof, we introduce a mathematical formalization for a common operation in AD libraries that manually sets the gradient of some computation to 0. 
    We show these conditions hold for a wide variety of gradient estimation methods for first order differentiation.
    For many score function-based methods like RELAX \cite{grathwohlBackpropagationVoidOptimizing2018}, MAPO \cite{liangMemoryAugmentedPolicy2019} and the unordered set estimator \cite{koolEstimatingGradientsDiscrete2020}, the conditions also hold for any-order differentiation. 
    In \emph{Storchastic}, we only have to prove these conditions locally. 
    % meaning that we can easily combine different gradient estimation methods for different sampling steps.
    This means that modelers are free to choose the gradient estimation method that best suits each sampling step, while guaranteeing that the gradient remains unbiased. 
    % They can then reuse previous implementations while preventing subtle implementation mistakes that lead to unexpected estimation bias.
    \emph{Storchastic} is the first stochastic AD framework to incorporate the measure-valued derivative \cite{pflugSamplingDerivativesProbabilities1989,heidergottMeasurevaluedDifferentiationMarkov2008,mohamedMonteCarloGradient2020} and SPSA \cite{spallMultivariateStochasticApproximation1992, bhatnagarStochasticRecursiveAlgorithms2013}, and the first to guarantee variance reduction of any-order derivative estimates through \tadditive{}s.

    % We implement the Storchastic framework as an open source library for PyTorch \cite{paszkePyTorchImperativeStyle2019}. It provides reference implementations for the discussed estimators with an easy-to-use API that automatically parallelizes different samples. 
    In short, our contributions are the following:
    \begin{enumerate}
        \item We introduce \emph{Storchastic}, a new framework for general stochastic AD that uses four gradient estimation components, in Section \ref{sec:requirements}-\ref{sec:surrogate}.
        \item We prove Theorem \ref{thrm:storchastic-informal}, which provides conditions under which \emph{Storchastic} gives unbiased any-order derivative estimates in Section \ref{sec:storchastic-conditions}. To this end, we introduce a mathematical formalization of forward-mode evaluation in AD libraries in Section \ref{sec:forward-mode}.
        % TODO: Insert theorem link?
        \item We derive a technique for extending variance reduction using \tadditive{}s to any-order derivative estimation in Section \ref{sec:var-reduction}.
        \item We implement \emph{Storchastic} as an open source library for PyTorch, Section \ref{sec:implementation}.
    \end{enumerate}

\section{Background}
We use capital letters $\node, \detnode, S_1, ..., S_k$ for nodes in a graph, calligraphic capital letters $\stochastic, \deterministic$ for sets and non-capital letters for concrete computable objects such as functions $f$ and values $\vals{i}$.
\subsection{Stochastic Computation Graphs}
We start by introducing Stochastic Computation Graphs (SCGs) \cite{schulmanGradientEstimationUsing2015}, which is a formalism for stochastic AD.
% \begin{definition}
A \textit{Stochastic Computation Graph} (SCG) is a directed acyclic graph (DAG) $\scg=\left(\nodes, \edges\right)$ where nodes $\nodes$ are partitioned in \textit{stochastic nodes} $\stochastic$ and \textit{deterministic nodes} $\deterministic$. 
We define the set of \textit{parameters} $\parameters\subseteq\deterministic$ such that  all $\theta\in\parameters$ have no incoming edges, and the set of \textit{cost nodes} $\costs\subseteq\deterministic$ such that all $c \in \costs$ have no outgoing edges. 

The set of \textit{parents} $\pa(\node)$ is the set of incoming nodes of a node $\node\in \nodes$, that is $\pa(\node)=\{\nodealt\in\nodes|(\nodealt, \node)\in \edges\}$. 
% Every node $\node\in\nodes$ has a \textit{(sample) space} $\sspace_\node$ that represents the set of possible outcomes at node $\node$.
Each \textit{stochastic node} $\stochnode\in\stochastic$ represents a random variable with \emph{sample space} $\sspace_\stochnode$ and probability distribution $p_\stochnode$ conditioned on its parents.
    Each \textit{deterministic node} $\detnode$ represents a (deterministic) function $\func_\detnode$ of its parents. 
    %  We additionally define for any positive integer $k$ the set $\deterministic^{(k)} \subseteq \deterministic$ as the set of nodes in $\deterministic$ for which the functions are $k$-times differentiable.
% \end{definition}
%   are a type of Bayesian Networks \cite{pearlProbabilisticReasoningIntelligent1988} with many deterministic (ie, degenerate) probability distributions. 
%

% SCGs are used to estimate expected values of cost nodes and gradients of the cost nodes with respect to the input parameters \cite{schulmanGradientEstimationUsing2015,foersterDiCEInfinitelyDifferentiable2018}. 
% Also define $\node$ \textit{influences} $\nodealt$ (denoted $\node\influences\nodealt$) if there is a directed path from $\node$ to $\nodealt$ in $\scg$. % sequence of nodes $\node=\node_0,\node_1, ..., \node_n=\nodealt$ such that for all $i=1, \dots, n$, $(\node_{i-1}, \node_i)\in\edges$, that is, there is a \textit{path} from $\node$ to $\nodealt$. 
$\nodealt$ \emph{influences} $\node$, denoted $\nodealt\influences \node$, if there is a directed path from $\nodealt$ to $\node$. 
We denote with $\nodes_{\before\node}=\{\nodealt\in\nodes|\nodealt\influences\node \}$ the set of nodes that influence $\node$.
% , and similarly with $\nodes_{\after\node}$ %=\{\nodealt|\node\influences\nodealt, \nodealt\in\nodes\}$  the nodes that $\node$ influences.
% Let $\bx_{\stochastic} \in\sspace_\stochastic$ be an \textit{outcome} of the stochastic nodes, where $\sspace_\stochastic=\prod_{\stochnode\in\stochastic}\sspace_\stochnode$. 
% % Next, we show how $\bx_\stochastic$ induces the outcome $\bx_\deterministic$ of the deterministic nodes. 
% Define $X(\detnode| \bx_\stochastic)$ which computes the outcome $x_\detnode$ in $\bx_\deterministic$ given the outcome of stochastic nodes $\bx_\stochastic$. 
% Next, define $\bx_{\pa(\detnode)}=x_{\node_1}, \dots, x_{\node_{|\pa(\detnode)|}}$ as the outcome of the parents of $\detnode$, where $x_{\node_i}$ is looked up in $\bx_\stochastic$ if $\node_i\in\stochastic$, and computed through $X(\node_i|\bx_\stochastic)$ otherwise. 
% By the DAG assumption, we find that $X(\detnode|\bx_\stochastic)=\func_\detnode(\bx_{\pa(\detnode)})$.
% \begin{definition}
    The \textit{joint probability} of all random variables $\bx_\stochastic\in\prod_{\stochnode\in\stochastic}\sspace_\stochnode$ is defined as 
    $p(\bx_\stochastic)=\prod_{\stochnode\in\stochastic}p_\stochnode(\bx_\stochnode|\bx_{\pa(\stochnode)})$, where $\vals{\pa(\stochnode)}$ is the set of values of the nodes $\pa(\stochnode)$.
    % We will often use $p(\bx_{\stochnode}|\vals{\stochastic_{\before \stochnode}})=p(\bx_{\stochnode}|\vals{\pa(\stochnode)})$ since by conditioning on all stochastic nodes that influence $\stochnode$ we can deterministically compute the values of the parent nodes of $\stochnode$.
% It should be noted that deterministic nodes can be seen as degenerate conditional probability distributions for which only a single outcome has positive probability mass or density. This makes SCGs a special case of Bayesian Networks, and thus inherits its conditional independence results. For the purpose of conceptual clarity, we treat deterministic nodes as mathematical functions instead.
    The \textit{expected value} of a deterministic node $\detnode\in\deterministic$ is its expected value over sampling stochastic nodes that influence that node, that is, 
    \begin{equation}
    \mathbb{E}[\detnode]=\mathbb{E}_{\stochastic_{\before \detnode}}[\func_\detnode(\pa(F))]=\int_{\sspace_{\stochastic_{\before\detnode}}} p(\vals{\before\stochastic}) \func_\detnode(\vals{\pa(\detnode)})d\vals{\stochastic_{\before \detnode}}.
    \end{equation}
% \end{definition}
% \begin{proposition}
%     Let  $\stochastic_{\before\detnode}=\nodes_{\before\detnode}\cap\stochastic$ be the set of stochastic nodes influencing $\detnode$. Then the expected value of a deterministic node $\detnode\in\deterministic$ can be computed by
%     \begin{equation}
%     \mathbb{E}[\detnode]=\mathbb{E}_{\stochnode\in\stochastic_{\influences\detnode}}[\func_\detnode(\pa(F))]=\sum_{\bx_{\stochastic_{\before\detnode}}\in\sspace_{\stochastic_{\before\detnode}}} p(\bx_{\stochastic_{\before\detnode}}) \func_\detnode(\bx_{\pa(\detnode)})
%     \end{equation}
%     where $\sspace_{\stochastic_{\before\detnode}}=\prod_{\stochnode\in\stochastic_{\before\detnode}}\sspace_\stochnode$.
% \end{proposition}
\subsection{Problem Statement}
\label{sec:problem-statement}
In this paper, we aim to define a \emph{surrogate loss} that, when differentiated using an AD library, gives an unbiased estimate of the $n$-th order derivative of a parameter $\theta$ with respect to the expected total cost $\nabla_\theta^{(n)}\mathbb{E}[\sum_{\cost\in\costs}\cost]$.
This gradient can be written as $\sum_{\cost\in\costs}\nabla_\theta^{(n)}\mathbb{E}[\cost]$, and we focus on estimating the gradient of a single cost node $\nabla_\theta^{(n)}\mathbb{E}[\cost]$. 
 
\subsection{Example: Discrete Variational Autoencoder}
\label{sec:example-vae}
\begin{figure}
    \centering
    \includegraphics[width=0.6\linewidth]{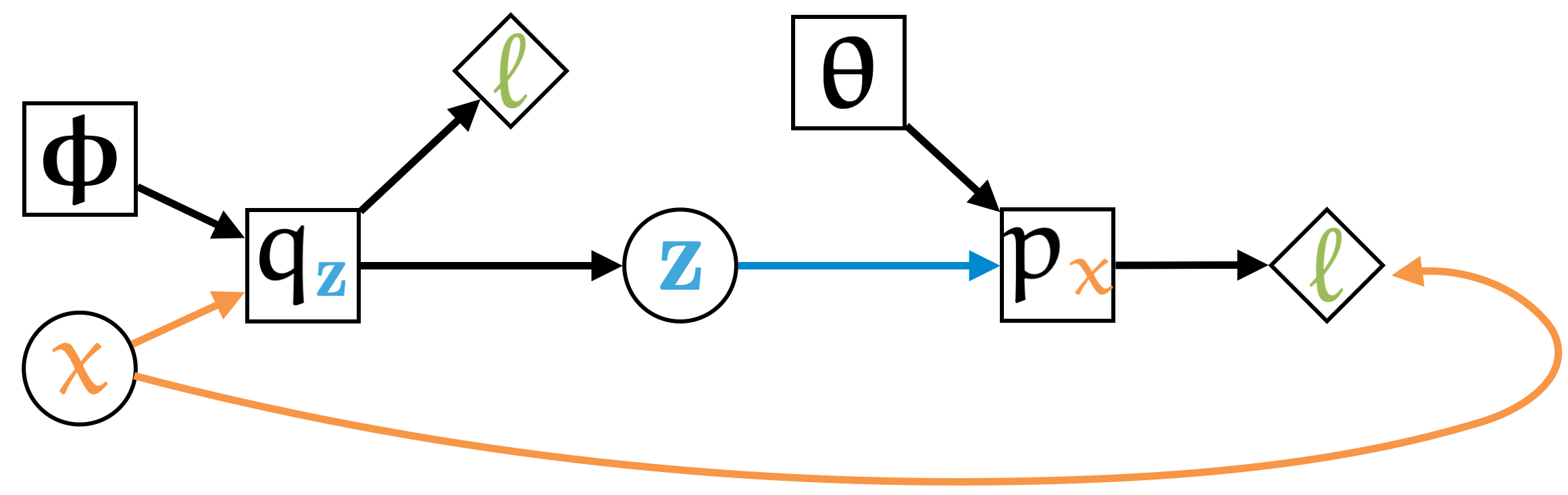}
    \caption{A Stochastic Computation Graph representing the computation of the losses of an VAE with a discrete latent space. }
    \label{fig:VAE}
\end{figure}
Next, we introduce a running example: A variational autoencoder (VAE) with a discrete latent space \cite{kingmaAutoencodingVariationalBayes2014, jangCategoricalReparameterizationGumbelsoftmax2017}. 
First, we represent the model as an SCG: The \textit{deterministic} nodes are $\mathcal{N}=\{ \phi, \theta,  q_{z}, p_{x}, \ell_{KLD}, \ell_{Rec}\}$ and the \textit{stochastic} nodes are $\mathcal{S}=\{x, z\}$. 
These are connected as shown in Figure \ref{fig:VAE}. 
The \textit{parameters} are $\Theta = \{\theta, \phi\}$ which respectively are the parameters of the variational posterior $q$ and the model likelihood $p$, and the \textit{cost nodes} $\mathcal{C}=\{\ell_{KLD}, \ell_{Rec}\}$ that represent the KL-divergence between the posterior and the prior, and the `reconstruction loss', or the model log-likelihood after decoding the sample $z$. 
Finally, $q_z$ represents the parameters of the multivariate categorical distribution of the amortized variational posterior $q_\phi(z|x)$.
This SCG represents the equation
\begin{equation}
    \mathbb{E}_{x, z}[\ell_{KLD} + \ell_{Rec}]
    % =\mathbb{E}_x[\ell_{KLD} + \mathbb{E}_{z\sim q_\phi(z|x)}[\ell_{Rec}]]
    = \mathbb{E}_x[\ell_{KLD}] + \mathbb{E}_{x, z\sim q_\phi(z|x)}[\ell_{Rec}].
\end{equation}
The problem we are interested in is estimating the gradients of these expectations with respect to the parameters. 
Since $x$ is not influenced by the parameters, we have $\nabla_\theta^{(n)} \mathbb{E}_x[\ell_{KLD}]=0$ and $\nabla_\phi^{(n)} \mathbb{E}_x[\ell_{KLD}]=\mathbb{E}_x[\nabla_\phi^{(n)} \ell_{KLD}]$. 
The second term is more challenging.
We can move the gradient with respect to $\theta$ in, since $z$ is not influenced by $\theta$: $\nabla_\theta^{(n)}\mathbb{E}_{x, z\sim q_\phi(z|x)}[\ell_{Rec}] = \mathbb{E}_{x, z\sim q_\phi(z|x)}[\nabla_\theta^{(n)}\ell_{Rec}]$.
However, we cannot compute $\nabla_\phi^{(n)}\mathbb{E}_{x, z\sim q_\phi(z|x)}[\ell_{Rec}]$ without gradient estimation methods. This is because sampling from $q_\phi(z|x)$ is dependent on $\phi$. 
Furthermore, since we are dealing with a discrete stochastic node, we cannot apply the reparameterization method here without introducing bias.

% \todo{\begin{enumerate}
%     \item Example of SCG
%     \item Stochastic nodes: Just $z$
%     \item Deterministic nodes: Encoder, decoder, prior (also cost), reconstruction loss
%     \item Construct problem statement with gradient
% \end{enumerate}}

\subsection{Formalizing AD libraries and DiCE}
\label{sec:forward-mode}
To be able to properly formalize and prove the propositions in this paper, we introduce the `forward-mode' operator that simulates forward-mode evaluation using AD libraries. 
This operator properly handles the common `stop-grad' operator, which ensures that its argument is only evaluated during forward-mode evaluations of the computation graph.
It is implemented in Tensorflow and Jax with the name \texttt{stop\_gradient} \cite{abadiTensorFlowLargescaleMachine2015,bradburyJAXComposableTransformations2018} and in PyTorch as \texttt{detach} or \texttt{no\_grad} \cite{paszkePyTorchImperativeStyle2019}. 
`stop-grad' is necessary to define surrogate losses for gradient estimation, which is why it is essential to properly define it.
For formal definitions of the following operators and proofs we refer the reader to Appendix \ref{sec:appendix-forward-mode}. 
\begin{definition}[informal]
    The \emph{stop-grad} operator $\bot$ is a function such that $\nabla_x\bot(x)=0$.
    The \emph{forward-mode} operator $\forward{}$, which is denoted as an arrow above the argument it evaluates, acts as an identity function, except that $\forward{\bot(a)}=\forward{a}$.
    Additionally, we define the \magicbox{} operator as $\magic(x)=\exp(x-\bot(x))$.
\end{definition}
Importantly, the definition of $\forward{}$ implies that $\forward{\nabla_x f(x)}$ does not equal $\nabla_x \forward{f(x)}$ if $f$ contains a stop-grad operator. 
\magicbox{}, which was first introduced in \cite{foersterDiCEInfinitelyDifferentiable2018}, is particularly useful for creating surrogate losses that remain unbiased for any-order differentiation. 
It is defined such that $\forward{\magic(x)}=1$ and $\nabla_x\magic(f(x))=\magic(f(x))\nabla_x f(x)$. 
This allows injecting multiplicative factors to the computation graph only when computing gradients.

Making use of \magicbox{}, DiCE \cite{foersterDiCEInfinitelyDifferentiable2018} is an estimator for automatic $n$th-order derivative estimation that defines a surrogate loss using the score function: 
\begin{equation}
    \label{eq:DiCE}
    \nabla_\theta^{(n)}\mathbb{E}[\sum_{\cost\in\costs}\cost] = \mathbb{E}\bigg[\forward{\nabla_\theta^{(n)} \sum_{\cost\in\costs}\magic\big(\sum_{\stochnode \in \stochastic_{\before\cost}}\log p(\vals{\stochnode}|\vals{\pa(\stochnode)})\big)\cost}\bigg].
\end{equation}
DiCE correctly handles the credit assignment problem: The score function is only applied to the stochastic nodes that influence a cost node.
It also handles pathwise dependencies of the parameter through cost functions.
However, it has high variance since it is based on a straightforward application of the score function.
% The variance of the estimator can be reduced with a baseline by adding $\sum_{\stochnode \in \stochastic}(1-\magic(\log p(\vals{\stochnode}|\vals{\pa(\stochnode)})))\baseline{\stochnode}(\vals{\pa(\stochnode)})$, and a baseline for higher-order derivative estimation also exists \cite{foersterDiCEInfinitelyDifferentiable2018} which we will also discuss in section INSERT SECTION.
% TODO Conditional independence is probably not relevant anymore
% \subsection{Conditional Independence}
% Because of the close resemblance of SCGs to Bayesian networks, they inherit many of their \textit{conditional independence} properties. Conditional independence of stochastic nodes $A$ and $B$ given another node $C$, denoted $(A\condind B) | C$, means that $P(A, B|C)=P(A|C)P(B|C)$. That is, knowing the value of $C$ makes $A$ and $B$ independently distributed. 

% \textbf{Local Markov property.} The Local Markov property of Bayesian networks says that a node $\node$, given the outcomes of parent nodes $\pa(\node)$, is conditionally independent of the nodes they do not influence $\nodes\setminus \nodes_{\node\influences}$. 

% This implies that $P(\node|\nodes\setminus \nodes_{\node\influences})=P(\node|\pa(\node))$

\section{The \emph{Storchastic} Framework}
\label{sec:storchastic}
In this section, we introduce \emph{Storchastic}, a framework for general any-order gradient estimation in SCGs that gives modelers the freedom to choose a suitable gradient estimation method for each stochastic node. 
First, we present 5 requirements that we used to develop the framework in Section \ref{sec:requirements}. 
\emph{Storchastic} deconstructs gradient estimators into four components that we present in Section \ref{sec:components}. %: The \tproposal{}, weighting function, \tmultipl{} and \tadditive{}.
We use these components to introduce the \emph{Storchastic} surrogate loss in Section \ref{sec:surrogate}, and give conditions that need to hold for unbiased estimation in Section \ref{sec:storchastic-conditions}.
In Section \ref{sec:var-reduction} we discuss variance reduction, in Section \ref{sec:estimators} we discuss several estimators that fit in \emph{Storchastic}, and in Section \ref{sec:implementation} we discuss our PyTorch implementation.
An overview of our approach is outlined in Figure \ref{fig:surrogate-loss}.

\subsection{Requirements of the \emph{Storchastic} Framework}
\label{sec:requirements}
First, we introduce the set of requirements we used to develop \emph{Storchastic}.
\begin{enumerate}
    \item \label{req:plug} Modelers should be able to choose a different gradient estimation method for each stochastic node. 
    This allows for choosing the method best suited for that stochastic node, or adding background knowledge in the estimator.
    \item \label{req:estimators} \emph{Storchastic} should be flexible enough to allow implementing a wide range of reviewed gradient estimation methods, including score function-based methods with complex sampling techniques \cite{yinARSMAugmentREINFORCESwapMergeEstimator2019,koolEstimatingGradientsDiscrete2020, liangMemoryAugmentedPolicy2019} or \tadditive{}s \cite{grathwohlBackpropagationVoidOptimizing2018, tuckerREBARLowvarianceUnbiased2017}, and other methods such as measure-valued derivatives \cite{heidergottMeasurevaluedDifferentiationMarkov2008,pflugSamplingDerivativesProbabilities1989} and SPSA \cite{spallMultivariateStochasticApproximation1992} which are missing AD implementations \cite{mohamedMonteCarloGradient2020}.
    \item \label{req:surrogate} \emph{Storchastic} should define a \emph{surrogate loss} \cite{schulmanGradientEstimationUsing2015}, which gives gradients of the SCG when differentiated using an AD library. This makes it easier to implement gradient estimation methods as modelers get the computation of derivatives for free.
    \item \label{req:higher-order} Differentiating the surrogate loss $n$ times should give estimates of the $n$th-order derivative, which are used in for example reinforcement learning \cite{furmstonApproximateNewtonMethods2016, foersterLearningOpponentLearningAwareness2018} and meta-learning \cite{finnModelAgnosticMetaLearningFast2017, liMetaSGDLearningLearn2017}.
    \item \label{req:variance} Variance reduction methods through better sampling and \tadditive{}s should generalize in higher-order derivative estimation.
    \item \label{req:unbiased} \emph{Storchastic} should be provably unbiased. To reduce the effort of developing new methods, researchers should only have to prove a set of local conditions that generalize to any SCG. 
\end{enumerate}

\subsection{Gradient Estimators in Storchastic}
\label{sec:components}
Next, we introduce each of the four components and motivate why each is needed to ensure Requirement \ref{req:estimators} is satisfied.
First, we note that several recent gradient estimators, like MAPO \cite{liangMemoryAugmentedPolicy2019}, unordered set estimator \cite{koolEstimatingGradientsDiscrete2020} and self-critical baselines \cite{koolAttentionLearnSolve2019,rennieSelfCriticalSequenceTraining2017} act on sequences of stochastic nodes instead of on a single stochastic node. 
Therefore, we create a partition $\stochastic_1, ..., \stochastic_k$ of $\stochastic_{\before\cost}$ topologically ordered by the influence relation, and define the shorthand $\vals{i}:= \vals{\stochastic_i}$.
For each partition $\stochastic_i$, we choose a \emph{gradient estimator}, which is a 4-tuple $\gradestim$. Here, $\fproposalcond{i}$ is the \emph{\tproposal{}}, $\fweight{i}$ is the \emph{weighting function}, $\fmultipl{i}$ is the \emph{\tmultipl{}} and $\additive{i}$ is the \emph{\tadditive{}}. 

% \subsubsection{\tproposal}
\subsubsection{Proposal distribution}
Many gradient estimation methods in the literature do not sample a single value $\vals{i}\sim p(\vals{i}|\vals{<i})$, but sample, often multiple, values from possibly a different distribution. 
Some instances of sampling schemes are taking multiple i.i.d. samples, importance sampling \cite{mahmoodWeightedImportanceSampling2014} which is very common in off-policy reinforcement learning, sampling without replacement \cite{koolEstimatingGradientsDiscrete2020}, memory-augmented sampling \cite{liangMemoryAugmentedPolicy2019} and antithetic sampling \cite{yinARMAugmentREINFORCEMergeGradient2019}.
Furthermore, measure-valued derivatives \cite{mohamedMonteCarloGradient2020,heidergottMeasurevaluedDifferentiationMarkov2008} and SPSA \cite{spallMultivariateStochasticApproximation1992} also sample from different distributions by comparing the performance of two related distributions.
To capture this, the \tproposal{} $\fproposalcond{i}$ samples a \emph{set} of values $\sampleset_i=\{\vals{i, 1}, ..., \vals{i, m}\}$ where each $\vals{i, j}\in \sspace_{\stochastic_i}$. 
% It samples a set of values, since most of the methods that use different sampling schemes sample multiple values to compute the gradient estimate.
The sample is conditioned on $\vals{<i}=\cup_{\stochnode\in\stochastic_i} \vals{\pa(\stochnode)}$, the values of the parent nodes of the stochastic nodes in $\stochastic_i$. 
This is illustrated in Figure \ref{fig:surrogate-loss}.b.

% This influence can be either direct through the measure of $\proposalcond{i}$, indirect, or not at all.  
% While all conditions constrain the choice of \tproposal{}, we here give the first:
% \begin{condition}
%     The measure of \tproposal{} should be equal to its forward-mode evaluation:
% \begin{equation}
%     \forward(\proposalcond{i}) = \proposalcond{i}
% \end{equation}
% \end{condition}
% This is a minor condition that is easy to satisfy by 
\subsubsection{Weighting function}
When a gradient estimator uses a different sampling scheme, we have to weight each individual sample to ensure it remains a valid estimate of the expectation. 
For this, we use a nonnegative weighting function $\weight{i}: \sspace_{\stochastic_i}\rightarrow \mathbb{R}^+$.  
Usually, this function is going to be detached from the computation graph, but we allow it to receive gradients as well to support implementing expectations and gradient estimation methods that compute the expectation over (a subset of) values \cite{koolEstimatingGradientsDiscrete2020, liangMemoryAugmentedPolicy2019,liuRaoBlackwellizedStochasticGradients2019}.

% \subsubsection{\tmultipl{}}

\subsubsection{Gradient function}
The \tmultipl{} is an unbiased gradient estimator together with the weighting function. 
It distributes the empirical cost evaluation to the parameters of the distribution. 
In the case of score function methods, this is the log-probability. For measure-valued derivatives and SPSA we can use the parameters of the distribution itself.

% \subsubsection{\tadditive{}}
\subsubsection{Control variate}
Modelers can use \tadditive{}s to reduce the variance of gradient estimates \cite{greensmithVarianceReductionTechniques2004,mohamedMonteCarloGradient2020}. 
It is a function that has zero-mean when differentiated.
Within the context of score functions, a common control variate is a baseline, which is a function that is independent of the sampled value. 
We also found that LAX, RELAX, and REBAR (Appendix \ref{sec:relax}), and the GO gradient \cite{congGOGradientExpectationbased2019} (Appendix \ref{sec:gogradient}) have natural implementations using a \tadditive{}. 
We discuss how we implement \tadditive{}s in \emph{Storchastic} in Section \ref{sec:var-reduction}.

\subsubsection{Example: Leave-one-out baseline}
\label{sec:example-loo}
As an example, we show how to formulate the score function with the leave-one-out baseline \cite{mnihVariationalInferenceMonte2016,koolBuyREINFORCESamples2019} in \emph{Storchastic}.
This method samples $m$ values with replacement and uses the average of the other values as a baseline. 
\begin{itemize}
    \item \textbf{Proposal distribution}: We use $m$ samples with replacement, which can be formulated as $q(\sampleset_{i}|\vals{<i})=\prod_{j=1}^m p(\vals{i, j}|\vals{<i})$.
    \item \textbf{Weighting function}: Since samples are independent, we use $\fweight{i}=\frac{1}{m}$. 
    \item \textbf{Gradient function}: The score-function uses the log-probability $\fmultipl{i}=\log p(\vals{i}|\vals{<i})$.
    \item \textbf{Control variate}: We use $\fadditive{i, j}=(1-\magic(\fmultipl{i}))\frac{1}{m-1}\sum_{ j'\neq j} f_\cost(\vals{<i}, \vals{i, j})$,
    where $\frac{1}{m-1}\sum_{ j'\neq j} f_\cost(\vals{<i}, \vals{i, j})$ is the leave-one-out baseline.
    $(1-\magic(\fmultipl{i}))$ is used to ensure the baseline will be subtracted from the cost before multiplication with the \tmultipl{}. 
    It will not affect the forward evaluation since $\forward{1-\magic(\fmultipl{i})}$ evaluates to 0.
\end{itemize}

\subsection{The \emph{Storchastic} Surrogate Loss}
\label{sec:surrogate}
As mentioned in Requirement \ref{req:surrogate}, we would like to define a \emph{surrogate loss}, which we will introduce next.
Differentiating this loss $n$ times, and then evaluating the result using an AD library, will give unbiased estimates of the $n$-th order derivative of the parameter $\theta$ with respect to the cost $\cost$.
Furthermore, according to Requirement \ref{req:plug}, we assume the modeler has chosen a gradient estimator $\gradestim$ for each partition $\stochastic_i$, which can all be different.
Then the Storchastic surrogate loss is
% \begin{align}
%     \nabla_\theta^{(n)} \mathbb{E}_{\stochastic_{\before \cost}}[\cost]=&\mathbb{E}_{\fproposal{1}} \bigg[ \sum_{\itersample{1}} 
%         \mathbb{E}_{\fproposalcond{2}}\bigg[\sum_{\itersample{2}}\dots \mathbb{E}_{\fproposalcond{k}} \bigg[ \sum_{\itersample{k}} \\
%         &\forward\Big(\nabla_\theta^{(n)} \prod_{i=1}^k \fweight{i} \cdot \Big( \sum_{i=1}^k\fadditive{i} + \magic\big( \sum_{i=1}^k \fmultipl{i} \big)\cost\Big)   \Big) \bigg] \dots \bigg] \bigg]
%
\begin{align}
    \SL= \sum_{\itersample{1}} \fweight{1}&\Big[\fadditive{1} + 
        \sum_{\itersample{2}} \fweight{2}\Big[ \magic(\fmultipl{1})\fadditive{2} + \dots  \\
        + \sum_{\itersample{k}} \fweight{k}  &\Big[ \magic(\sum_{j=1}^{k-1} \fmultipl{j})\fadditive{k} + \magic\big( \sum_{i=1}^k \fmultipl{i} \big)\cost  \Big] \dots \Big] \Big]   \label{eq:surrogate-loss},\\
    \text{where } \sampleset_1 \sim \fproposal{1}&, \sampleset_2 \sim \fproposalcond{2}, ..., \sampleset_k \sim \fproposalcond{k}.
\end{align}
When this loss is differentiated $n$ times using AD libraries, it will produce unbiased estimates of the $n$-th derivative, as we will show later.
\begin{algorithm}
\begin{algorithmic}[1]
    \Function{estimate\_gradient}{$n$, $\theta$} 
        \State $\SL \gets $ \Call{surrogate\_loss}{$1$, $\{\}$, 0} \Comment{Compute surrogate loss} 
        \State \Return $\forward{\nabla_\theta^{(n)} \SL}$ \Comment{Differentiate and use AD library to evaluate surrogate loss}
    \EndFunction
    \State
    \Function{surrogate\_loss}{$i$, $\vals{<i}$, $L$} 
        \If {$i=k+1$}
            \State \Return $\magic\big(L) f_\cost(\vals{\leq k})$ \Comment{Use \magicbox{} to distribute cost}
            % \State \Return $L$
        \EndIf
        \State $\sampleset_{i} \sim \fproposalcond{i}$ \Comment{Sample from \tproposal{}}
        \State $\mathtt{sum} \gets 0$
        \For {$\itersample{i}$} \Comment{Iterate over options in sampled set }
            \State $A \gets \magic(L)\fadditive{i}$ \Comment{Compute \tadditive{}}
            \State $L_i\gets L + \fmultipl{i}$ \Comment{Compute gradient function} 
            \State $\hat{\cost} \gets$ \Call{surrogate\_loss}{$i+1$, $\vals{\leq i}$, $L_i$} \Comment{Compute surrogate loss for $\vals{i}$}
            \State $\mathtt{sum} \gets \mathtt{sum} + \fweight{i}(\hat{\cost} + A)$ \Comment{Weight and add}
        \EndFor
        \State \Return $\mathtt{sum}$
    \EndFunction
\end{algorithmic}
\caption{The \emph{Storchastic} framework: Compute a Monte Carlo estimate of the $n$-th order gradient given $k$ gradient estimators $\gradestim$. }
\label{alg:storchastic}
\end{algorithm}
% \begin{algorithm}
% \begin{algorithmic}[1]
%     \Function{estimate\_gradient}{$n$, $\theta$} 
%         \State $\SL \gets $ \Call{surrogate\_loss}{$1$, $\{\}$, 0} \Comment{Compute surrogate loss} 
%         \State \Return $\forward{\nabla_\theta^{(n)} \SL}$ \Comment{Differentiate and use AD library to evaluate surrogate loss}
%     \EndFunction
%     \State
%     \Function{surrogate\_loss}{$i$, $\vals{<i}$, $L$} 
%         \If {$i=k+1$}
%             \State \Return $\magic\big(L) f_\cost(\vals{\leq k})$ \Comment{Use \magicbox{} to distribute cost}
%             % \State \Return $L$
%         \EndIf
%         \State $\sampleset_{i} \sim \fproposalcond{i}$ \Comment{Sample from \tproposal{}}
%         \State $sum \gets 0$
%         \For {$\itersample{i}$} \Comment{Iterate over options in sampled set }
%             \State $A \gets \magic(L)\fadditive{i}$ \Comment{Compute \tadditive{}}
%             \State $L_i\gets L + \fmultipl{i}$ \Comment{Compute gradient function} 
%             \State $\hat{\cost} \gets$ \Call{surrogate\_loss}{$i+1$, $\vals{\leq i}$, $L_i$} \Comment{Compute surrogate loss for $\vals{i}$}
%             \State $sum \gets sum + \fweight{i}(\hat{\cost} + A)$ \Comment{Weight and add}
%         \EndFor
%         \State \Return $sum$
%     \EndFunction
% \end{algorithmic}
% \caption{The \emph{Storchastic} framework: Compute a Monte Carlo estimate of the $n$-th order gradient given $k$ gradient estimators $\gradestim$. }
% \label{alg:storchastic}
% \end{algorithm}
To help understand the \emph{Storchastic} surrogate loss and why it satisfies the requirements, we will break it down using Algorithm \ref{alg:storchastic}. 
The \textproc{estimate\_gradient} function computes the surrogate loss for the SCG, and then differentiates it $n\geq 0$ times using the AD library to return an estimate of the $n$-th order gradient, which should be unbiased according to Requirement \ref{req:higher-order}.
If $n$ is set to zero, this returns an estimate of the expected cost.

The \textproc{surrogate\_loss} function computes the equation using a recursive computation, which is illustrated in Figure \ref{fig:surrogate-loss}.b-d.
It iterates through the partitions and uses the gradient estimator to sample and compute the output of each component.
It receives three inputs: The first input $i$ indexes the partitions and gradient estimators, the second input $\vals{<i}$ is the set of previously sampled values for partitions $\stochastic_{<i}$, and $L$ is the sum of \tmultipl{}s of those previously sampled values.
In line 8, we sample a set of values $\sampleset_i$ for partition $i$ from $\fproposalcond{i}$.
In lines 9 to 14, we compute the sum over values $\vals{i}$ in $\sampleset_i$, which reflects the $i$-th sum of the equation.
Within this summation, in lines 11 and 12, we compute the \tmultipl{} and \tadditive{} for each value $\vals{i}$.
We will explain in Section \ref{sec:var-reduction} why we multiply the \tadditive{} with the \magicbox{} of the sum of the previous \tmultipl{}.

In line 13, we go into recursion by moving to the next partition. 
We condition the surrogate loss on the previous samples $\vals{<i}$ together with the newly sampled value $\vals{i}$.
We pass the sum of \tmultipl{}s for later usage in the recursion. 
Finally, in line 14, the sample performance and the \tadditive{} are added in a weighted sum. 
% In the recursion, this same process is repeated, but conditioned on  sampled values $\vals{<i}$ of the previous partitions $\stochastic_1, ..., \stochastic_{i-1}$.
The recursion call happens for each $\itersample{i}$, meaning that this computation is exponential in the size of the sampled sets of values $\sampleset_i$.
For example, the surrogate loss samples $|\sampleset_1|$ times from $\proposalcond{2}$, one for each value $\itersample{1}$.
% We see in the first line of the equation that we sum over the values in each set of sampled values $\sampleset_{i}$, 
% which corresponds to the summation in lines 5 to 10 of the algorithm.  
However, this computation can be trivially parallelized by using tensor operations in AD libraries.
An illustration of this parallelized computation is given in Figure \ref{fig:surrogate-loss}.

Finally, in line 7 after having sampled values for all $k$ partitions, we compute the cost, and multiply it with the \magicbox{} of the sum of \tmultipl{}s.
This is similar to what happens in the DiCE estimator in Equation \eqref{eq:DiCE}.
\emph{Storchastic} can be extended to multiple cost nodes by computing surrogate losses for each cost node, and adding these together before differentiation.
For stochastic nodes that influence multiple cost nodes, the algorithm can share samples and gradient estimation methods to reduce overhead.
% Note that the samples are taken conditioned on the currently iterated previous values $\vals{<i}$.

% The second line of the equation uses the gradient estimators to compute the loss term for the currently sampled values. 
% The loss term sums all \tadditive{}s and \tmultipl{}s.
% The latter are processed with the \magicbox{}, which ensures pathwise gradients through $\cost$ properly flow, and that the estimation will work for any order of differentiation, including zero-th order.
% Each loss term is weighted by the product over weight functions.

\subsection{Conditions for Unbiased Estimation}
\label{sec:storchastic-conditions}
We next introduce our main result that shows \emph{Storchastic} satisfies Requirements \ref{req:higher-order} and \ref{req:unbiased}, namely the conditions the gradient estimators should satisfy such that the \emph{Storchastic} surrogate loss gives estimates of the $n$-th order gradient of the SCG. 
A useful part of our result is that, in line with Requirement \ref{req:unbiased}, only local conditions of gradient estimators have to be proven to ensure estimates are unbiased.
Our result gives immediate generalization of these local proofs to any SCG.
\begin{manualtheorem}{1}
   \label{thrm:storchastic-informal}
   Evaluating the $n$-th order derivative of the \emph{Storchastic} surrogate loss in Equation \eqref{eq:surrogate-loss} using an AD library is an unbiased estimate of $\nabla_\theta^{(n)} \mathbb{E}[\cost]$ under the following conditions. First, all functions $f_\detnode$ corresponding to deterministic nodes $\detnode$ and all probability measures $p_\stochnode$ corresponding to stochastic nodes $\stochnode$ are \emph{identical under evaluation}. 
   Secondly, for each gradient estimator $\gradestim$, $i=1, ..., k$, all the following hold for $m=0, ..., n$:
   \begin{enumerate}
       \item  $\mathbb{E}_{\fproposalcond{i}}[\sum_{\itersample{i}} \forward{\nabla^{(m)}_\theta \fweight{i} \magic(\fmultipl{i})f(\vals{i})}]= \forward{\nabla_\theta^{(m)} \mathbb{E}_{\stochastic_i}[f(\vals{i})] }$ for any deterministic function $f$;
       \item $\mathbb{E}_{\fproposalcond{i}}[\sum_{\itersample{i}} \forward{\nabla^{(m)}_\theta \fweight{i} \fadditive{i}}]=0$;
       \item  for $n\geq m>0$, $\mathbb{E}_{\fproposalcond{i}}[\sum_{\itersample{i}} \forward{\nabla_\theta^{(m)} \fweight{i}}]=0$;
       \item $\forward{\fproposalcond{i}} = \fproposalcond{i}$, for all permissible $\sampleset_i$.
   \end{enumerate}
\end{manualtheorem}
The first condition defines a local surrogate loss for single expectations of any function under the \tproposal{}. 
% This surrogate loss weights each function evaluation using the weighting function, and multiplies with the \magicbox{} of the \tmultipl{}.
The condition then says that this surrogate loss should give an unbiased estimate of the gradient for all orders of differentiation $m=0, ..., n$. 
Note that since 0 is included, the forward evaluation should also be unbiased.
This is the main condition used to prove unbiasedness of the \emph{Storchastic} framework, and can be proven for the score function and expectation, and for measure-valued derivatives and SPSA for zeroth and first-order differentiation. 

The second condition says that the \tadditive{} should be 0 in expectation under the \tproposal{} for all orders of differentiation. 
This is how \tadditive{}s are defined in previous work \cite{mohamedMonteCarloGradient2020}, and should usually not restrict the choice. 
The third condition constrains the weighting function to be 0 in expectation for orders of differentiation larger than 0.
Usually, this is satisfied by the fact that weighting functions are detached from the computation graph, but when enumerating expectations, this can be shown by using that the sum of weights is constant.
The final condition is a regularity condition that says \tproposal{}s should not be different under forward mode.
We also assume that the SCG is \emph{identical under evaluation}. This means that all functions and probability densities evaluate to the same value with and without the forward-mode operator, even when differentiated. 
This concept is formally introduced in Appendix \ref{sec:appendix-forward-mode}.

A full formalization and the proof of Theorem \ref{thrm:storchastic} are given in Appendix \ref{seq:unbiasedness-proof}. 
The general idea is to rewrite each sampling step as an expectation, and then inductively show that the inner expectation $i$ over the \tproposal{} $\proposalcond{i}$ is an unbiased estimate of the $n$th-order derivative over $\stochastic_i$ conditional on the previous samples.
To reduce the multiple sums over \tmultipl{}s inside \magicbox{}, we make use of a property of \magicbox{} proven in Appendix \ref{sec:appendix-forward-mode}:
\begin{manualproposition}{\ref{prop:DiCE_multiply}}
   Summation inside a \magicbox{} is equivalent under evaluation to multiplication of the arguments in individual \magicbox{}es, ie:
   \begin{equation}
    \magic(l_1(x) + l_2(x)) f(x) \equivforward  \magic(l_1(x)) \magic(l_2(x)) f(x).
   \end{equation}
\end{manualproposition}
\emph{Equivalence under evaluation}, denoted $\equivforward$, informally means that, under evaluation of $\forward{}$, the two expressions and their derivatives are equal.
This equivalence is closely related to $e^{a+b}=e^a e^b$.

\subsection{Any-order variance reduction using \tadditive{}s}
\label{sec:var-reduction}
% \tadditive{}s
% Control variates are a common variance reduction technique for gradient estimators \cite{greensmithVarianceReductionTechniques2004,mohamedMonteCarloGradient2019}. 
To satisfy Requirement \ref{req:variance}, we investigate implementing \tadditive{}s such that the variance of any-order derivatives is properly decreased.
This is challenging in general SCG's \cite{maoBaselineAnyOrder2019}, since in higher orders of differentiation, derivatives of \tmultipl{}s will interact, but naive implementations of \tadditive{}s only reduce the variance of the \tmultipl{} corresponding to a single stochastic node.
\emph{Storchastic} solves this problem similarly to the method introduced in \cite{maoBaselineAnyOrder2019}. 
In line 11 of the algorithm, we multiply the \tadditive{} with the sum of preceding \tmultipl{}s $\magic(L)$. 
We prove that this ensures every term of the any-order derivative will be affected by a \tadditive{} in Appendix \ref{sec:baselines}. 
This proof is new, since \cite{maoBaselineAnyOrder2019} only showed this for first and second order differentiation, not for general \tadditive{}s, and uses a slightly different formulation that we show misses some terms.

\begin{manualtheorem}{2}[informal]
Let $\Multipl{i}=\sum_{j=1} ^i \multipl{i}$. The \emph{Storchastic} surrogate loss of \eqref{eq:surrogate-loss} can equivalently be computed as
\begin{align}
    \SL\equivforward \sum_{\itersample{1}} &  
        \sum_{\itersample{2}}   \dots  
         \sum_{\itersample{k}} \prod_{i=1}^k \fweight{i}  \sum_{i=1}^k \magic(\Multipl{i-1})\Big(\fadditive{i} + (\magic(\multipl{i}) - 1)\cost \Big) + \cost.
    % \text{where } \sampleset_1 \sim \fproposal{1}&, \sampleset_2 \sim \fproposalcond{2}, ..., \sampleset_k \sim \fproposalcond{k}
\end{align}
\end{manualtheorem}
This gives insight into how \tadditive{}s are used in \emph{Storchastic}. 
They are added to the \tmultipl{}, but only during differentiation since $\forward{\magic(\Multipl{i}) - 1} = 0$. 
Furthermore, since both terms are multiplied with $\magic(\Multipl{i-1})$ (see line 11 of Algorithm \ref{alg:storchastic}), both terms correctly distribute over the same any-order derivative terms. 
By choosing a \tadditive{} of the form $\fadditive{i}=(1 - \magic(\multipl{i}))\cdot  \baseline{i}$, we recover baselines which are common in the context of score functions \cite{foersterDiCEInfinitelyDifferentiable2018, mohamedMonteCarloGradient2020}.
For the proof, we use the following proposition also proven in Appendix \ref{sec:baselines}:
\begin{manualproposition}{\ref{prop:baseline-generator}}
    For orders of differentiation $n>0$, 
    \begin{equation}
        \label{eq:baseline-generator}
        \forward{\nabla_\node^{(n)} \magic(\Multipl{k})} = \forward{\nabla_\node^{(n)} \sum_{i=1}^k \big(\magic(\multipl{i}) - 1\big) \magic(\Multipl{i-1})}.
    \end{equation}
\end{manualproposition}

\subsection{Gradient Estimation Methods}
\label{sec:estimators}
In Appendix \ref{sec:gradient-estimators} we show how several prominent examples of gradient estimation methods in the literature can be formulated using \emph{Storchastic}, and prove for what orders of differentiation the conditions hold.
Starting off, we show that for finite discrete random variables, we can formulate enumerating over all possible options using \emph{Storchastic}. 
The score function fits by mimicking DiCE \cite{foersterDiCEInfinitelyDifferentiable2018}. 
We extend it to multiple samples with replacement to allow using the leave-one-out baseline \cite{mnihVariationalInferenceMonte2016,koolBuyREINFORCESamples2019}. 
Furthermore, we show how importance sampling, sum-and-sample estimators such as MAPO \cite{liangMemoryAugmentedPolicy2019}, the unordered set estimator \cite{koolEstimatingGradientsDiscrete2020} and RELAX and REBAR \cite{grathwohlBackpropagationVoidOptimizing2018, tuckerREBARLowvarianceUnbiased2017} fit in \emph{Storchastic}.
We also discuss the antithetic sampling estimator ARM \cite{yinARMAugmentREINFORCEMergeGradient2019}.
Unfortunately, condition 2 only holds for this estimator for $n\leq 1$ since it relies on a particular property of the score function that holds only for first-order gradient estimation. 
In addition to score function based methods, we discuss the GO gradient, SPSA \cite{rubinsteinSimulationMonteCarlo2016} and Measure-Valued Derivative \cite{heidergottMeasurevaluedDifferentiationMarkov2008}, and show that the last two will only be unbiased for $n\leq 1$. 
Finally, we note that reparameterization \cite{kingmaAutoencodingVariationalBayes2014, rezendeStochasticBackpropagationApproximate2014} can be implemented by transforming the SCG such that the sampling step is outside the path from the parameter to the cost \cite{schulmanGradientEstimationUsing2015}.
% \subsection{Sampling methods
% \begin{itemize}
%     \item Compare different ways of sampling in literature review
% \end{itemize}

\subsection{Implementation}
\label{sec:implementation}
We implemented \emph{Storchastic} as an open source PyTorch \cite{paszkePyTorchImperativeStyle2019} library
\footnote{Code is available at \url{github.com/HEmile/storchastic}.}. 
% \footnote{Code is available in the supplementary material and will be made publicly available for the camera ready. }. 
To ensure modelers can easily use this library, it automatically handles sets of samples as extra dimensions to PyTorch tensors which allows running multiple sample evaluations in parallel.
This approach is illustrated in Figure \ref{fig:surrogate-loss}.
By making use of PyTorch broadcasting semantics, this allows defining models for simple single-sample computations that are automatically parallelized using \emph{Storchastic} when using multiple samples.
The \emph{Storchastic} library has implemented most of the gradient estimation methods mentioned in Section \ref{sec:estimators}.
Furthermore, new gradient estimation methods can seamlessly be added.

\subsubsection{Example: Leave-one-out baseline in Discrete Variational Autoencoder}
\begin{figure}
\begin{lstlisting}[language=Python]
class ScoreFunctionLOO(storch.method.Method):
    def proposal_dist(self, distribution, amt_samples):
        return distr.sample((amt_samples,))

    def weighting_function(self, distribution, amt_samples):
        return torch.full(amt_samples, 1/amt_samples)

    def estimator(self, sample, cost):
        # Compute gradient function (log-probability)
        log_prob = sample.distribution.log_prob(tensor)
        sum_costs = storch.sum(costs.detach(), sample.name)
        # Compute control variate
        baseline = (sum_costs - costs) / (sample.n - 1)
        return log_prob, (1.0 - magic_box(log_prob)) * baseline
\end{lstlisting}
\caption{Implementing the score function with the leave-one-out baseline in the Storchastic library.}
\label{fig:list-loo}
\end{figure}
As a small case study, we show how to implement the score function with the leave-one-out baseline introduced in Section \ref{sec:example-loo} for the discrete variational autoencoder introduced in Section \ref{sec:example-vae} in PyTorch using Storchastic. While the code listed is simplified, it shows the flexibility with which one can compute gradients in SCGs. 

We list in Figure \ref{fig:list-loo} how to implement the score function with the leave-one-out baseline.
Line 3 implements the proposal distribution, line 6 the weighting function, line 10 the gradient function and line 13 and 14 the control variate.
Gradient estimation methods in Storchastic all extend a common base class \texttt{storch.method.Method} to allow easy interoperability between different methods.

In Figure \ref{fig:list-vae}, we show how to implement the discrete VAE. The implementation directly follows the SCG shown in Figure \ref{fig:VAE}. In line 2, we create the \texttt{ScoreFunctionLOO} method defined in Figure \ref{fig:list-loo}. 
Then, we run the training loop: In line 6 we create the stochastic node $x$ by denoting the minibatch dimension as an independent dimension.
In line 8 we run the encoder with parameters $\phi$ to find the variational posterior $q_z$. We call the gradient estimation method in line 9 to get a sample of $z$. Note that this interface is independent of gradient estimation method chosen, meaning that if we wanted to compare our implemented method with a baseline, all that is needed is to change line 2. After the decoder, we compute the two costs in lines 12 and 13. Finally, we call Storchastic main algorithm in line 15 and run the optimizer.
\begin{figure}
\begin{lstlisting}[language=Python]
from vae import minibatches, encode, decode, KLD, binary_cross_entropy
method = ScoreFunctionLOO("z", 8)
for data in minibatches():
    optimizer.zero_grad()
    # Denote minibatch dimension as independent plate dimension
    data = storch.denote_independent(data.view(-1, 784), 0, "data")
    # Compute variational distribution given data, sample z
    q = torch.distributions.OneHotCategorical(logits=encode(data))
    z = method(q)
    # Compute costs, form the ELBO
    reconstruction = decode(z)
    storch.add_cost(KLD(q))
    storch.add_cost(binary_cross_entropy(reconstruction, data))
    # Storchastic backward pass, optimize
    ELBO = storch.backward()
    optimizer.step()
\end{lstlisting}
\caption{Simplified implementation of the discrete VAE using Storchastic.}
\label{fig:list-vae}
\end{figure}

We run this model on our currently implemented set of gradient estimation methods for discrete variables in Appendix \ref{sec:experiments} and report the results, which are meant purely to illustrate the case study.  
% \end{minipage}\hfill
% \begin{minipage}{.48\textwidth}

% \end{minipage}
% \end{figure}
% \begin{itemize}
%     \item Can be efficiently implemented by rewriting everything as tensors
%     \item Storchastic automatically handles input dimensions to handle summing the tensors correctly
% \end{itemize}

% \subsection{Discrete VAE}

% \subsection{Sigmoid Belief Network/other experiment}

\section{Related Work}
\label{sec:related-work}
The literature on gradient estimation is rich, with papers focusing on general methods that can be implemented in \emph{Storchastic} \cite{spallMultivariateStochasticApproximation1992,heidergottMeasurevaluedDifferentiationMarkov2008, grathwohlBackpropagationVoidOptimizing2018,yinARMAugmentREINFORCEMergeGradient2019,liangMemoryAugmentedPolicy2019, liuRaoBlackwellizedStochasticGradients2019, congGOGradientExpectationbased2019}, see Appendix \ref{sec:gradient-estimators}, and works focused on Reinforcement Learning \cite{williamsSimpleStatisticalGradientfollowing1992,lillicrapContinuousControlDeep2016, mnihAsynchronousMethodsDeep2016} or Variational Inference \cite{mnihVariationalInferenceMonte2016}. For a recent overview, see \cite{mohamedMonteCarloGradient2020}. 

The literature focused on SCGs is split into methods using reparameterization \cite{rezendeStochasticBackpropagationApproximate2014,kingmaAutoencodingVariationalBayes2014,figurnovImplicitReparameterizationGradients2018,maddisonConcreteDistributionContinuous2017, jangCategoricalReparameterizationGumbelsoftmax2017} and those using the score function \cite{schulmanGradientEstimationUsing2015}. Of those, DiCE \cite{foersterDiCEInfinitelyDifferentiable2018} is most similar to \emph{Storchastic}, and can do any-order estimation on general SCGs. 
DiCE is used in the probabilistic programming library Pyro \cite{binghamPyroDeepUniversal2019}. 
We extend DiCE to allow for incorporating many other gradient estimation methods than just basic score function. 
We also derive and prove correctness of a general implementation for control variates for any-order estimation which is similar to the one conjectured for DiCE in \cite{maoBaselineAnyOrder2019}.

\cite{parmasTotalStochasticGradient2018,xuBackpropQGeneralizedBackpropagation2019} and \cite{weberCreditAssignmentTechniques2019} study actor-critic-like techniques and bootstrapping for SCGs to incorporate reparameterization using methods inspired by deterministic policy gradients \cite{lillicrapContinuousControlDeep2016}. By using models to differentiate through, these methods are biased through model inaccuracies and thus do not directly fit into \emph{Storchastic}.
However, combining these ideas with the automatic nature of \emph{Storchastic} could be interesting future work. 

% $\lambda_S$ \cite{shermanComputableSemanticsDifferentiable2021} solves the automatic differentiation of expectations by explicitly treating integrals as higher-order functions, foregoing the need of sampling and gradient estimation. This is complex and hard to scale, but exact. For discrete stochastic nodes, recent work investigated efficient marginalization of the expectation using sparse marginalization \cite{niculaeSparseMAPDifferentiableSparse2018, correiaEfficientMarginalizationDiscrete2020}. 

\section{Conclusion}
  We investigated general automatic differentiation for stochastic computation graphs. 
    We developed the \emph{Storchastic} framework, and introduced an algorithm for unbiased any-order gradient estimation that allows using a large variety of gradient estimation methods from the literature. 
    We also investigated variance reduction and showed how to properly implement \tadditive{}s such that it affects any-order gradient estimates.
    The framework satisfies the requirements introduced in Section \ref{sec:requirements}. 

    For future work, we are interested in extending the analysis of \emph{Storchastic} to how variance compounds when using different gradient estimation methods. 
    Furthermore, \emph{Storchastic} could be extended to allow for biased methods.
    We are also interested in closely analyzing the different components of gradient estimators, both from a theoretical and empirical point of view, to develop new estimators that combine the strengths of estimators in the literature.
    
% TODO: Add back in preprint and camera ready
% \section*{Acknowledgement}
% The authors would like to thank Peter Bloem, Bernd Heidergott and Wouter Kool for extensive and useful discussions.

% \begin{enumerate}
  
%     \item we investigated general gradient estimation for stochastic computation graphs 
%     \item Storchastic allows for using a large variety of gradient estimation methods, unlike previous work, and is available for higher-order estimation and variance reduction. 
%     \item Future work: Variance analysis: Do local decreases in variance always result in total lower variance? 
%     \item We've limited ourselves to unbiased estimators. What biased estimators can also be implemented in Storchastic? How does bias accumulate?
%     \item Future work: It would be interesting to see how the individual components can be combined to further reduce the variance of estimators.
%     \item Future work: Clear comparison between different methods.
% \end{enumerate}
% Proof: Reparameterization unbiasedness conditions: If we have an unbiased gradient estimate for the sampled value x, we can do reparameterization. Ie, if all paths are differentiable and either end in a stochastic node of which the estimation is unbiased (assumption) or in a cost node.

% d-separation: Each variable is conditionally independent of its non-descents given its parents. 
\bibliographystyle{abbrvnat}
\bibliography{references.bib}

\appendix

\newcommand{\sampledice}{\bx_{\stochnode, i_{\stochnode}}}
\newcommand{\gradest}{g\left(i, \left\{ \sample, \costresult_i \right\}_{i=1}^\amt, \bx_{\stochastic_{\before\stochnode}} \right)}
\newcommand{\sampleprob}{\mu_\stochnode(\bx_\stochnode|\bx_{\pa(\stochnode)})}
\newcommand{\sampleprobof}[1]{\mu_\stochnode(\bx_\stochnode=#1|\bx_{\pa(\stochnode)})}
\newcommand{\functionAPP}{\func_\detnode(\bx_{\pa(\detnode)})}

\section{Forward-mode evaluation}
\label{sec:appendix-forward-mode}
In this section, we define several operators that we will use to mathematically define operators used within deep learning
to implement gradient estimators.

To define these, we will need to distinguish how deep learning libraries evaluate their functions.
~\cite{foersterDiCEInfinitelyDifferentiable2018} handles this using a different kind of equality, denoted $\mapsto$.
Unfortunately, it is not formally introduced, making it unclear as to what rules are allowed with this equality.
For instance, they define the DiCE operator as
\[
\begin{array}{l}\text { 1. }\magic(f(\theta)) \mapsto 1 \\ \text { 2. } \nabla_{\theta} \magic(f(\theta))=\magic(f(\theta)) \nabla_\theta f(\theta)\end{array}
\]
However, without a clearly defined meaning of $\mapsto$ `equality under evaluation', it is unclear whether the following is allowed:
\[
    \nabla_{\theta} \magic(f(\theta)) \mapsto \nabla_{\theta} 1 = 0
\]
This would lead to a contradiction, as by definition
\[
    \nabla_{\theta} \magic(f(\theta)) = \magic(f(\theta)) \nabla_\theta f(\theta) \mapsto \nabla_\theta f(\theta)
\].
We first introduce an unambigiuous formulation for forward mode evaluation that does not allow such inconsistencies.

\begin{definition}
    The \emph{stop-grad} operator $\bot$ is a function such that $\nabla_x\bot(x)=0$.
    The \emph{forward-mode} operator $\forward{}$ is a function such that, for well formed formulas $a$ and $b$,
    \begin{enumerate}
        \item $\forward{\bot(a)}=\forward{a}$
        \item $\forward{a+b}=\forward{a}+\forward{b}$
        \item $\forward{a\cdot b}=\forward{a} \cdot \forward{b}$
        \item $\forward{a^b}=\forward{a}^{ \forward{b}}$
        \item $\forward{c}=c$, if $c$ is a constant or a variable.
        \item $\forward{\forward{a}}=\forward{a}$
    \end{enumerate}
    Additionally, we define the DiCE operator $\magic(x)=\exp(x-\bot(x))$
\end{definition}
When computing the results of a function $f(x)$, Deep Learning libraries instead compute $\forward{f(x)}$.
Importantly, $\forward{\nabla_x f(x)}$ does not always equal $\nabla_x\forward{f(x)}$.
For example, $\forward{\nabla_x \bot(f(x))}=\forward{0}=0$, while $\nabla_x \forward{\bot(f(x))}=\nabla_x \forward{f(x)}$.

In the last example, the derivative will first have to be rewritten to find a closed-form formula that does not contain the $\forward{}$ operator.
Furthermore, $\bot(f(x))$ only evaluates to a closed-form formula if it is reduced using derivation, or if it is enclosed in $\forward{}$.

We note that $\mathbb{E}_{p(x)}[\forward{f(x)}]=\forward{\mathbb{E}_{p(x)}[f(x)]}$ for both continuous and discrete distributions $p(x)$ if $\forward{p(x)}=p(x)$. This is easy to see for discrete distributions since these are weighted sums over an amount of elements. For continuous distributions we can use the Riemann integral definition. 
\begin{align}
    \mathbb{E}_{p(x)}[\forward{f(x)}] = \int p(x) \forward{f(x)}
\end{align}
\begin{proposition}
    1: $\forward{\magic(f(x))}=1$ and 2: $\nabla_x \magic(f(x))=\magic(f(x))\cdot \nabla_x f(x)$
\end{proposition}
\begin{proof}
    \begin{enumerate}
    \item 
        \begin{align*}
            \forward{\magic(f(x))} &= \forward{\exp(f(x)-\bot(f(x)))} \\
            &= \exp(\forward{f(x)}-\forward{\bot(f(x))})\\
            &= \exp(f(x)-f(x))=1
        \end{align*}
    \item
        \begin{align*}
        \nabla_x \magic(f(x)) &= \nabla_x\exp(f(x)-\bot(f(x))) \\
            &= \exp(f(x)-\bot(f(x))) \nabla_x (f(x)-\bot(f(x)))\\
            &= \magic(f(x))( \nabla_x f(x) - \nabla_x \bot(f(x)) \nabla_x f(x)) \\
            &= \magic(f(x)) (\nabla_x f(x) - 0\cdot\nabla_x f(x)) = \magic(f(x)) \nabla_x f(x)
        \end{align*}
    \end{enumerate} 
\end{proof}
Furthermore, unlike in the DiCE paper, with this notation $\forward{\nabla_x \magic(f(x))}$ unambiguously evaluates to $\forward{\nabla_x f(x)}$, as
$\forward{\nabla_x \magic(f(x))}=\forward{\magic(f(x))\nabla_x f(x)}=\forward{\magic(f(x))}\cdot \forward{\nabla_x f(x)}=\forward{\nabla_x f(x)}$.
Note that, although this is not a closed-form formula, by finding a closed-form formula for $\nabla_x f(x)$, this can be
reduced to $\nabla_x f(x)$.

\begin{proposition}
    \label{prop:DICE}
    For any two functions $f(x)$ and $l(x)$, it holds that  for all $n\in (0, 1, 2, ...)$,
    \begin{equation}
        \forward {\nabla^{(n)}_{x} \magic( l(x) f(x) ) } = \forward{ \g^{(n)}(x)}.
    \end{equation}
    where $\g^{(n)}(x)=\nabla_x \g^{(n-1)}(x)+\g^{(n-1)}(x)\nabla_x l(x)$ for $n>0$, and $\g^{(0)}(x)=f(x)$.
\end{proposition}

For this proof, we use a similar argument as in \cite{foersterDiCEInfinitelyDifferentiable2018}.

\begin{proof}
    % We first define $\gmagic^{(n)}(\vals{i})=\magic(\multipl{i}) \g^{(n)}(\vals{i})$, and note that by proposition TODO $\forward( \g^{(n)}(\vals{i}) + \nabla^{(n)}_\node \weight{i} \additive{i})= \forward( \magic(\multipl{i})\g^{(n)}(\vals{i}) + \nabla^{(n)}_\node \weight{i} \additive{i}) = \forward( \gmagic^{(n)}(\vals{i}) + \nabla^{(n)}_\node \weight{i} \additive{i})$. We can thus instead show by induction that 
    
    % \begin{equation}
    %     \forward\Big(\nabla^{(n)}_{\node}\weight{i} \big(\additive{i} + \magic( \multipl{i}) f(\vals{i}) \big) \Big) = \forward( \gmagic^{(n)}(\vals{i}) + \nabla^{(n)}_\node \weight{i} \additive{i}) 
    % \end{equation}

    First, we show that $\magic(l(x)) \g^{(n)}(x) = \nabla_x^{(n)}\magic(l(x))f(x)$.     We start off with the base case, $n=0$. Then, %$\magic(\multipl{i}) \g^{(0)}(\vals{i}) = \magic(\multipl{i}) \weight{i}f(\vals{i})$ such that 
    $\magic(l(x))\g^{(0)}(x) = \magic(l(x))f(x) $.

    Next, assume the proposition holds for $n$, that is, $\magic(l(x)) \g^{(n)}(x) = \nabla_x^{(n)} \magic(\multipl{i})f(x)$. Consider $n+1$. 

    \begin{align}
        \magic(l(x)) \g^{(n+1)}(x) &= \magic(l(x)) ( \nabla_x \g^{(n)}(x) + g^{(n)}(x)\nabla_x l(x)) \\
        &= \nabla_x \magic(l(x))  \g^{(n)}(x) \\
        &= \nabla_x( \nabla_x^{(n)}(\magic(l(x))f(x)) ) \\
        &= \nabla_x^{(n + 1)}\magic(l(x))f(x)
    \end{align}

    Where from line 1 to 2 we use the DiCE proposition in the reversed direction. From 2 to 3 we use the inductive hypothesis.

    %  We start off with the base case, $n=0$. Then, $\forward( \g^{(0)}(\vals{i})) = \forward(\weight{i}f(\vals{i}))$ such that $\forward(\g^{(n)}(\vals{i}) + \weight{i}\additive{i}) = \forward(\weight{i}f(\vals{i})) = \forward\Big(\weight{i} \big(\additive{i} + \magic( \multipl{i}) f(\vals{i}) \big) \Big)$.

    % Next, assume the proposition holds for $n$. Consider $n+1$. First, we proof that  $\magic(\multipl{i})\g^{(n)}=$
    
    We use this result, $\magic(l(x)) \g^{(n)}(x) = \nabla_x^{(n)}\magic(l(x))f(x)$, to prove our proposition. Since $\forward{a}=1\cdot \forward{a}=\forward{\magic(\multipl{i})}\forward{a}=\forward{\magic(\multipl{i})a}$,
    \begin{align}
        \forward {\g^{(n)}(x)} = \forward{ \magic(l(x)) \g^{(n)}(x) }= \forward {\nabla_\node^{(n)} \magic(l(x))f(x) } 
    \end{align}
\end{proof}
\begin{definition}
    We say a function $f$ is \emph{identical under evaluation} if for all $n\in (0, 1, 2, ...)$, $\forward{\nabla_x^{(n)} f(x)} = \nabla_x^{(n)} f(x)$. Furthermore, we say that two functions $f$ and $g$ are \emph{equivalent under evaluation}, denoted $f\equivforward g$, if for all $n\in (0, 1, 2, ...)$, $\forward{\nabla_x^{(n)} f(x)} = \forward{\nabla_x^{(n)} g(x)}$.
\end{definition}

Every function that does not contain a stop-grad operator ($\bot$) is identical under evaluation, although  functions that are identical under evaluation can have stop-grad operators (for example, consider $f(x)\equivforward f(x) + \bot(f(x) - f(x))$). Note that $\forward{f(x)}=f(x)$ does not necessarily mean that $f$ is identical under evaluation, since for instance the function $f'(x)=\magic(2x)f(x)$ has $\forward{f'(x)}=f(x)$, but $\forward{\nabla_x f'(x)} = \forward{\magic(2x)(\nabla_x f(x) + 2)} = \nabla_x f(x) + 2 \neq \nabla_x f'(x) = \magic(2x)(\nabla_x f(x) + 2)$.

\begin{proposition}
    \label{prop:equiv-eval}
    If $f(x)$ and $l(x)$ are identical under evaluation, then all $\g^{(n)}(x)$  from $n=0, ..., n$ as defined in Proposition \ref{prop:DICE} are also identical under evaluation. 
\end{proposition}
\begin{proof}
    Consider $n=0$. Then $g^{(0)}(x) = f(x)$. Since $f(x)$ is identical under evaluation, $g^{(0)}$ is as well. 
    
    Assume the proposition holds for $n$, and consider $n+1$. Let $m$ be any positive number. $\forward{\nabla_x^{(m)} g^{(n+1)}(x)}=\forward{\nabla_x^{(m)}(\nabla_x \g^{(n)}(x) + \g^{(n)}(x) \nabla_x l(x))}$. Since $\g^{(n)}(x)$ is identical under evaluation by the inductive hypothesis, $\forward{\nabla_x^{(m)}\nabla_x \g^{(n)}(x)} = \forward{\nabla_x^{(m+1)}\g^{(n)}(x)}=\nabla_x^{(m+1)}\g^{(n)}(x)$. 
    
    Next, using the general Leibniz rule, we find that $\forward{\nabla_x^{(m)}\g^{(n)}(x) \nabla_x l(x)}=\sum_{j=0}^m{m \choose j\nabla_x^{(m-j)}\g^{(n)}(x)}\forward{\nabla_x^{(j+1)}l(x)}$. Since both $\g^{(n)}(x)$ and $l(x)$ are identical under evaluation, this is equal to $\sum_{j=0}^m{m \choose j} \nabla_x^{(m-j)}\g^{(n)}(x)\nabla_x^{(j+1)}l(x)) = \nabla_x^{(m)}\g^{(n)}(x) \nabla_x l(x)$. 
    
    Therefore, $\forward{\nabla_x^{(m)} g^{(n+1)}(x)} = \nabla_x^{(m)}(\nabla_x \g^{(n)}(x) + \g^{(n)}(x) \nabla_x l(x)) = \nabla_x^{(m)}\g^{(n+1)}(x)$, which shows that $\g^{(n+1)}(x)$ is identical under evaluation. 
\end{proof}

We next introduce a very useful proposition that we will use to prove unbiasedness of the \emph{Storchastic} framework. This result was first used without proof in \cite{farquharLoadedDiCETrading2019}.

\begin{proposition}
    \label{prop:DiCE_multiply}
    For any three functions $l_1(x)$, $l_2(x)$ and $f(x)$, $\magic(l_1(x)+l_2(x)) f(x) \equivforward \magic(l_1(x)) \magic(l_2(x)) f(x)$. That is, for all $n\in (0, 1, 2, ...)$.

    \begin{equation}
        \forward{\nabla_x^{(n)} \magic(l_1(x)+l_2(x)) f(x)} = \forward{\nabla_x^{(n)} \magic(l_1(x)) \magic(l_2(x)) f(x)}
    \end{equation}
\end{proposition}
\begin{proof}
    Start with the base case $n=0$. Then, $\forward{\magic(l_1(x)+l_2(x))f(x)} = \forward{f(x)} = 1\cdot 1\cdot \forward{f(x)} = \forward{\magic(l_1(x))\magic(l_2(x))f(x)}$.

    Next, assume the proposition holds for $n$. Then consider $n+1$:
    \begin{align}
        &\forward{\nabla_x^{(n+1)}\magic(l_1(x))\magic(l_2(x))f(x)} \\
        =& \forward{\nabla_x^{(n)}\nabla_x \magic(l_1(x)\magic(l_2(x))f(x))} \\
        =& \forward {\nabla_x^{(n)} \magic(l_1(x))\magic(l_2(x))(f(x)\nabla_x l_1(x) + f(x)\nabla_x l_2(x) + \nabla_x f(x) )} \\
        =& \forward {\nabla_x^{(n)} \magic(l_1(x))\magic(l_2(x))(f(x)\nabla_x (l_1(x)+l_2(x)) +  \nabla_x f(x) )}
    \end{align}

    Define function $h(x) = f(x)\nabla_x (l_1(x)+l_2(x)) +  \nabla_x f(x)$. Since the proposition works for any function, we can apply the inductive hypothesis replacing $f(x)$ by $h(x)$:
    \begin{align}
     \forward {\nabla_x^{(n)} \magic(l_1(x))\magic(l_2(x)) h(x)}  \forward {\nabla_x^{(n)} \magic(l_1(x)+l_2(x)) h(x) } 
    \end{align}
    Finally, we use Proposition \ref{prop:DICE} with $g^{(1)}(x)=h(x)$ and $l(x)=l_1(x)+l_2(x)$:
    \begin{align}
         &\forward {\nabla_x^{(n)} \magic(l_1(x)+l_2(x)) f(x)\nabla_x (l_1(x)+l_2(x)) + \nabla_x f(x) }  \\ 
         =&\forward{\nabla_x^{(n)}\nabla_x \magic(l_1(x)+l_2(x)) f(x) } =\forward{\nabla_x^{(n+1)}\magic(l_1(x)+l_2(x)) f(x) }
    \end{align}
\end{proof}

It should be noted that it cannot be proven that $\nabla_x^{(n)} \magic(\sum_{i=1}^k l_i(x)) f(x) = \nabla_x^{(n)} \magic(\sum_{i=1}^{k-1} l_i(x)) \magic(l_k(x)) f(x)$ because the base-case cannot be proven without the $\forward{}$ operator interpreting the $\magic$ operator.

Also note the parallels with the exponential function, where $e^{l_1(x)+l_2(x)}=e^{l_1(x)} e^{l_2(x)}$.

\section{The \emph{Storchastic} framework (formal)}\label{sec:dice-formulation}

In this section we formally introduce \emph{Storchastic} to provide the mathematical machinery needed to prove our results.
Let $\stochastic_1, \dots, \stochastic_k$ be a partition of $\stochastic_{\before\detnode}$. 
Assume the sets $\stochastic_1, \dots, \stochastic_k$ are topologically sorted, that is, there is no $i<j$ such that there exists a stochastic node $\stochnode\in\stochastic_j$ that is also in $\stochastic_{<i}=\bigcup_{j=1}^{j-1} \stochastic_j$.
We use assignment $\vals{i}$ to denote a set that gives a value to each of the random variables $\stochnode \in \stochastic_i$. That is, $\vals{i}\in \prod_{\stochnode \in \stochastic_i}\sspace_\stochnode$. We additionally use $\vals{< i}$ to denote a set that gives values to all random variables in $\stochastic_{< i}$. 
In the same vein, $\sample{i}$ denotes a set of sets of values $\vals{i}$, that is $\sample{i}=\{\vals{i, 1}, ..., \vals{i, |\sample{i}|}\}$.

\begin{definition}
For each partition $\stochastic_i$ there is a \textbf{gradient estimator} $\gradestim$ where $\fproposalcond{i}$ is a distribution over a set of values $\sample{i}$ conditioned on $\vals{<i}$, $w_i:\prod_{\stochnode \in \stochastic_i}\sspace_\stochnode \rightarrow \mathbb{R}^+ $ is the weighting function that weights different values $\vals{i}$, $l_i:\prod_{\stochnode \in \stochastic_i}\sspace_\stochnode \rightarrow \mathbb{R}$ is the \tmultipl{} that provides the gradient produced by each $\vals{i}$, and the \tadditive{} $a_i:\prod_{j=1}^{i}\prod_{\stochnode \in \stochastic_j}\sspace_\stochnode \rightarrow \mathbb{R}$ is a function of both $\vals{i}$ and $\vals{<i}$.
\end{definition}

$\fproposalcond{i}$ is factorized as follows: Order stochastic nodes $\stochnode_{i, 1}, \dots, \stochnode_{i, m} \in \stochastic_i$ topologically, then
$\fproposalcond{i}=\prod_{j=1}^m q(\sampleset_{i, j}|\sampleset_{i, <j}, \vals{<i})$.

%\begin{align}
%    \mathbb{E}[\detnode]&=\mathbb{E}_{\sampleset_{\stochastic_1}, \dots, \sampleset_{\stochastic_k}\sim q} \Bigg[ \sum_{\bx_{\stochastic_1}\in \sampleset_{\stochastic_1}} \dots \sum_{\bx_{\stochastic_k} \in \sampleset_{\stochastic_k}} \left( \prod_{i=1}^k w_i(\bx_{\stochastic_j}) \right) \cdot \\
%    &\left( \magic\left( \sum_{i=1}^k l_i(\bx_{\stochastic_i}) \right) \detnode + \sum_{i=1}^k (1-\magic(l_i(\bx_{\stochastic_i}))) b_i(\bx_{\stochastic_i}) + a_i(\bx_{\stochastic_i})-\bot(a_i(\bx_{\stochastic_i})) \right) \Bigg]
%\end{align}
%
%(This isn't correct: The sets $\sampleset_{\stochastic_i}$ are different depending on the chosen values $\bx_{\stochastic_{<i}}$, yet here it is suggested that there is only one such set. The computational complexity is equal though, it's just that it should be a different set.)

In the rest of this appendix, we will define some shorthands to declutter the notation, as follows:
\begin{itemize}
    \item $\weight{i} = \fweight{i}$ and $\Weight{i} = \prod_{j=1}^i \weight{i}$
    \item $\multipl{i} = \fmultipl{i}$ and $\Multipl{i} = \sum_{j=1}^i \multipl{i}$
    \item $\additive{i} = \fadditive{i}$% and $\Additive{i} = \sum_{j=1}^i \magic(\Multipl{i-1}) \additive{i}$
    \item $\proposal{1} = \fproposal{1}$ and $\proposalcond{i} = \fproposalcond{i}$ (for $i > 1$)
\end{itemize}
These interfere with the functions and distributions themselves, but it should be clear from context which of the two is meant.

\begin{proposition}
\label{prop:translate-surrogate-loss}
Given a topologically sorted partition $\stochastic_1, ..., \stochastic_k$ of $\stochastic_{\before\detnode}$ and corresponding gradient estimators $\langle q_i, w_i, l_i, a_i \rangle$ for each $1\leq i\leq k$, the evaluation of the $n$-th order derivative of the \emph{Storchastic} surrogate loss $\forward{\nabla_\node^{(n)} \SL}$ of Equation \ref{eq:surrogate-loss} is equal in expectation to
\begin{align}
    \label{eq:storchastic-expectation}
   \mathbb{E}_{\proposal{1}} \bigg[ \sum_{\itersample{1}} \forward{\nabla_\node^{(n)} \weight{1} \additive{1}} + 
        \dots \mathbb{E}_{\proposalcond{k}} \bigg[ \sum_{\itersample{k}}\forward{ \nabla_\node^{(n)} \Weight{k}   
             \magic( \Multipl{k-1} )\additive{k} } +  \forward{\nabla_\node^{(n)} \Weight{k}   
             \magic( \Multipl{k} )  \detnode  } \bigg] \dots \bigg]
\end{align}
where the $i$-th term in the dots is $\mathbb{E}_{\proposalcond{i}}[\sum_{\itersample{i}}\forward{\nabla_\node^{(n)} \Weight{i} \magic{\Multipl{i-1}}\additive{i}} + (\dots)]$
\end{proposition}
\begin{proof}
    By moving the weights inwards and using the $\Multipl{i}$ notation,
    \begin{align}
        \forward{\nabla_\node^{(n)} \SL}=&\forward{\nabla_\node^{(n)} \sum_{\itersample{1}} \weight{1}\Big[\additive{1} + 
         \dots   
        + \sum_{\itersample{k}} \weight{k}  \Big[ \magic( \Multipl{k-1})\additive{k} + \magic(  \Multipl{k} )\cost  \Big] \dots \Big] \Big] } \\
        &=\forward{\nabla_\node^{(n)} \sum_{\itersample{1}}  \weight{1}\additive{1} +  
         \dots \sum_{\itersample{i}} \Weight{i} \magic(\Multipl{i-1}) \additive{i} + \dots}\\
        &\forward{+ \sum_{\itersample{k}} \Weight{k} \magic( \Multipl{k-1})\additive{k} + \Weight{k}\magic(  \Multipl{k} )\cost    } \\
        &= \sum_{\itersample{1}} \forward{\nabla_\node^{(n)} \weight{1}\additive{1}} +  
         \dots \sum_{\itersample{i}} \forward{\nabla_\node^{(n)}\Weight{i} \magic(\Multipl{i-1}) \additive{i}} + \dots\\
        &+ \sum_{\itersample{k}} \forward{\nabla_\node^{(n)}\Weight{k} \magic( \Multipl{k-1})\additive{k}} + \forward{\nabla_\node^{(n)}\Weight{k}\magic(  \Multipl{k} )\cost}   
    \end{align}
    This is all under sampling $\sampleset_1 \sim \fproposal{1}, \sampleset_2 \sim \fproposalcond{2}, ..., \sampleset_k \sim \fproposalcond{k}$. Taking expectations over these distributions before the respective summation over $\sampleset{i}$ gives the result.
\end{proof}

% We mean with the $\dots$ the following sequential computation:

% \begin{align}
%     \mathbb{E}_{\fproposal{1}} \bigg[\sum_{\itersample{1}} \mathbb{E}_{\fproposalcond{2}} \Big[ \sum_{\itersample{2}} \dots \mathbb{E}_{\fproposalcond{k-1}} \big[ \sum_{\itersample{k-1}} \mathbb{E}_{\fproposalcond{k}} [ () ] \big] \dots \Big] \bigg]
% \end{align}

In the \emph{Storchastic} framework, we require that $\mathbb{E}[F]$ is identical under evaluation, that is,  $\forward{\nabla_\node^{(n)} \mathbb{E}[F]}=\nabla_\node^{(n)} \mathbb{E}[F]$. This in practice means that the probability distributions and functions in the stochastic computation graph contain no stop gradient operators ($\bot$).

% !!!!!!!!!!!!!!!!!!!!!!!!!!!!!!!!!!!!!!!!!!!!!!!!!!!!!!!!!!!!!!!!!!!!!!!!!!!!
% !!!!!!!!!!!!!!!!!!!!!!!!!!!!!!!!!!!!!!!!!!!!!!!!!!!!!!!!!!!!!!!!!!!!!!!!!!!!
% !!!!!!!!!!!!!!!!!!!!!!!!!!!!!!!!!!!!!!!!!!!!!!!!!!!!!!!!!!!!!!!!!!!!!!!!!!!!
% TODO: ADAPT TO MOST RECENT FORMULATION INCLUDING VARIANCE REDUCTION
% !!!!!!!!!!!!!!!!!!!!!!!!!!!!!!!!!!!!!!!!!!!!!!!!!!!!!!!!!!!!!!!!!!!!!!!!!!!!
% !!!!!!!!!!!!!!!!!!!!!!!!!!!!!!!!!!!!!!!!!!!!!!!!!!!!!!!!!!!!!!!!!!!!!!!!!!!!
% !!!!!!!!!!!!!!!!!!!!!!!!!!!!!!!!!!!!!!!!!!!!!!!!!!!!!!!!!!!!!!!!!!!!!!!!!!!!
Using Proposition \ref{prop:DICE}, we give a recursive expression for $\forward{\nabla^{(n)}_{\node}\weight{i} \big(\additive{i} + \magic( \multipl{i}) f(\vals{i}) \big) }$.

\begin{proposition}
    \label{prop:grad_recur}
    For any gradient estimator $\gradestim$ it holds that  
    \begin{equation}
        \forward{\nabla^{(n)}_{\node}\weight{i} \big(\magic(\Multipl{i-1})\additive{i} + \magic( \multipl{i}) f(\vals{i}) \big) } = \forward{ \nabla^{(n)}_\node \weight{i} \magic(\Multipl{i-1})\additive{i} + \g_i^{(n)}(\vals{i})}
    \end{equation}
     where $\g_i^{(n)}(\vals{i})=\nabla_\node \g_i^{(n-1)}(\vals{i})+\g_i^{(n-1)}\nabla_\node \multipl{i}$ for $n>0$, and $\g_i^{(0)}(\vals{i})=\weight{i}f(\vals{i})$.
\end{proposition}

\begin{proof}
    Using Proposition \ref{prop:DICE}, we find that 
    \begin{align}
        \forward { \nabla_\node^{(n)} \weight{i}\magic(\Multipl{i-1})\additive{i} + \g_i^{(n)}(\vals{i}) } &= \forward {\nabla_\node^{(n)} \weight{i} \magic(\Multipl{i-1}) \additive{i}} + \forward{\nabla_\node^{(n)} \weight{i} \magic(\multipl{i})f(\vals{i}) }  \\
        &= \forward {\nabla_\node^{(n)} \weight{i} (\magic(\Multipl{i-1})\additive{i} + \magic(\multipl{i})f(\vals{i}) )} 
    \end{align}
\end{proof}

Proposition \ref{prop:grad_recur} is useful because it gives a fairly simple recursion to proof unbiasedness of any-order estimators with, when the gradient estimator is implemented in \emph{Storchastic}. Note that it doesn't itself show that such gradient estimators are unbiased in any-order derivatives.

% !!!!!!!!!!!!!!!!!!!!!!!!!!!!!!!!!!!!!!!!!!!!!!!!!!!!!!!!!!!!!!!!!!!!!!!!!!!!
% !!!!!!!!!!!!!!!!!!!!!!!!!!!!!!!!!!!!!!!!!!!!!!!!!!!!!!!!!!!!!!!!!!!!!!!!!!!!
% !!!!!!!!!!!!!!!!!!!!!!!!!!!!!!!!!!!!!!!!!!!!!!!!!!!!!!!!!!!!!!!!!!!!!!!!!!!!
% !!!!!!!!!!!!!!!!!!!!!!!!!!!!!!!!!!!!!!!!!!!!!!!!!!!!!!!!!!!!!!!!!!!!!!!!!!!!
% !!!!!!!!!!!!!!!!!!!!!!!!!!!!!!!!!!!!!!!!!!!!!!!!!!!!!!!!!!!!!!!!!!!!!!!!!!!!
% !!!!!!!!!!!!!!!!!!!!!!!!!!!!!!!!!!!!!!!!!!!!!!!!!!!!!!!!!!!!!!!!!!!!!!!!!!!!

\subsection{Unbiasedness of the \emph{Storchastic} framework}
\label{seq:unbiasedness-proof}
In this section, we use the equivalent expectation from Proposition \ref{prop:translate-surrogate-loss}
\begin{theorem}
    \label{thrm:storchastic}
    Let $\gradestim$ for $i=1, ..., k$ be a sequence of gradient estimators. Let the stochastic computation graph $\mathbb{E}[F]$ be identical under evaluation\footnote{In other words, all deterministic functions, and all probability measures associated with the stochastic nodes are identical under evaluation.}. The evaluation of the $n$th-order derivative of the \emph{Storchastic} surrogate loss is an unbiased estimate of $\nabla_\node^{(n)}\mathbb{E}[F]$ , that is
    \begin{align}
        \nabla_\node^{(n)} \mathbb{E}[\detnode] =
   \mathbb{E}_{\proposal{1}} \bigg[ \sum_{\itersample{1}} \forward{\nabla_\node^{(n)} \weight{1} \additive{1}} + 
        \dots \mathbb{E}_{\proposalcond{k}} \bigg[ \sum_{\itersample{k}}\forward{ \nabla_\node^{(n)} \Weight{k}   
             \magic( \Multipl{k-1} )\additive{k} } +  \forward{\nabla_\node^{(n)} \Weight{k}   
             \magic( \Multipl{k} )  \detnode  } \bigg] \dots \bigg]
    \end{align}
    if the following conditions hold for all estimators $i=1, ..., k$ and all preceding orders of differentiation $n\geq m \geq 0$:
    \begin{enumerate}
        \item  $\mathbb{E}_{q_i}[\sum_{\itersample{i}} \forward{\nabla^{(m)}_\node \weight{i} \magic(\multipl{i})f(\vals{i})}]= \forward{\nabla_\node^{(m)} \mathbb{E}_{\stochastic_i}[f(\vals{i})] }$ for any deterministic function $f$;
        \item $\mathbb{E}_{q_i}[\sum_{\itersample{i}} \forward{\nabla^{(m)}_\node \weight{i} \additive{i}}]=0$;
        \item  for $n\geq m>0$, $\mathbb{E}_{q_i}[\sum_{\itersample{i}} \forward{\nabla_\node^{(m)} \weight{i}}]=0$;
        \item $\forward{\fproposalcond{i}} = \fproposalcond{i}$.
    \end{enumerate}
\end{theorem}

% Condition one constrains the choice of the \tmultipl{} and the weighting function, which together should form an unbiased estimate of the $n$-th order gradient. 
% The second condition constrains the \tadditive{} to be 0 in expectation for any-order differentiation. The third condition constrains the weighting function to also be zero in expectation. This is often satisfied by the fact that for most estimators, $\sum_{\itersample{i}} \forward{\weight{i}}=1$. Note furthermore that from condition 1 follows that $\mathbb{E}_{q_i}[\sum_{\itersample{i}} \forward{\weight{i}}]=1$ by setting $f(\vals{i})=1$ and $m=0$. Finally, the fourth condition constrains the proposal condition to be equal to its forward mode evaluation. This is a minor condition needed to ensure the forward mode operator can pass through the expectations.

\begin{proof}
    In this proof, we make extensive use of the general Leibniz rule, which states that 
    \begin{equation}
        \nabla_x^{(n)} f(x) g(x) = \sum_{m=0}^n {n \choose m} \nabla_x^{(n-m)} f(x) \nabla_x^{(m)} g(x).
    \end{equation}
    We consider the terms  $\mathbb{E}_{\proposalcond{i}}\bigg[\sum_{\itersample{k}} \forward { \nabla_\node^{(n)} \Weight{i} \magic(\Multipl{i-1})\additive{i} }\bigg]$ and the term $\mathbb{E}_{\proposalcond{k}}\bigg[\sum_{\itersample{k}} \forward { \nabla_\node^{(n)} \Weight{k} \magic(\Multipl{k}) \detnode }\bigg]$ separately, starting with the first.

    \begin{lemma}
        \label{lemma:additive}
        For any positive number $1\leq j \leq k$,
        \begin{align}
            % &\mathbb{E}_{\proposal{1}}\bigg[\sum_{\itersample{1}}
                % \dots \mathbb{E}_{\proposalcond{j}}\bigg[\sum_{\itersample{j}} \forward { \nabla_\node^{(n)} \Weight{j} \Additive{j}  }\bigg]\dots \bigg]\\
            \mathbb{E}_{\proposal{1}}\bigg[\sum_{\itersample{1}}
                \dots \mathbb{E}_{\proposalcond{j}}\bigg[\sum_{\itersample{j}} \forward { \nabla_\node^{(n)} \Weight{j} \magic(\Multipl{j}) \additive{j}  }\bigg]\dots \bigg] = 0 .
        \end{align}
    \end{lemma}
    \begin{proof}

    We will prove the lemma using induction. First, let $j=1$. Then, using condition 2,
    \begin{align}
        \mathbb{E}_{\proposal{1}}\bigg[\sum_{\itersample{1}} \forward { \nabla_\node^{(n)} \weight{1} \additive{1}  }\bigg] = 0
    \end{align}
    Next, assume the inductive hypothesis holds for $j$, and consider the inner expectation of $j+1$:
    \begin{align}
        % & \mathbb{E}_{\proposalcond{j+1}}\bigg[\sum_{\itersample{j+1}} \forward{ \nabla_\node^{(n)} \Weight{j+1} \Additive{j+1}}\bigg] \\
        =& \mathbb{E}_{\proposalcond{j+1}}\bigg[\sum_{\itersample{j+1}} \forward{ \nabla_\node^{(n)}\magic(\Multipl{j})\additive{j+1} \Weight{j+1} }\bigg] = \mathbb{E}_{\proposalcond{j+1}}\bigg[\sum_{\itersample{j+1}} \forward{ \nabla_\node^{(n)}\weight{j+1}\additive{j+1} \magic(\Multipl{j})\Weight{j} }\bigg]\\
        =& \mathbb{E}_{\proposalcond{j+1}}\bigg[\sum_{\itersample{j+1}} \forward{ \sum_{m=0}^n \binom{n}{m}\nabla_\node^{(m)}\weight{j+1}  \additive{j+1} \nabla_\node^{(n-m)} \magic(\Multipl{j})\Weight{j}  }\bigg]
    \end{align}
    Next, note that $\Weight{j}$ and $\Additive{j}$ are both independent of $\vals{j+1}$. 
    Therefore, they can be moved out of the expectation. To do this, we implicitly use condition 4 to move the $\forward{}$ operator through the expectation. 
    \begin{align}
        & \mathbb{E}_{\proposalcond{j+1}}\bigg[\sum_{\itersample{j+1}} \sum_{m=0}^n \binom{n}{m}  \forward { \nabla_\node^{(m)} \weight{j+1}  \additive{j+1} \nabla_\node^{(n-m)} \magic(\Multipl{j})\Weight{j} } \bigg] \\
        =& \sum_{m=0}^n \binom{n}{m} \forward{\nabla_\node^{(n-m)} \magic(\Multipl{j})\Weight{j}} \mathbb{E}_{\proposalcond{j+1}}\Big[\sum_{\itersample{j+1}}   \forward { \nabla_\node^{(m)} \weight{j+1} \additive{j+1}   }\Big]
    \end{align}
    By condition 2 of the theorem, $\mathbb{E}_{\proposalcond{j+1}}\Big[\sum_{\itersample{j+1}}   \forward { \nabla_\node^{(m)} \weight{j+1} \additive{j+1}   }\Big]=0$. Therefore, we can remove this term and conclude that 
    \begin{align}
        \mathbb{E}_{\proposal{1}}\bigg[\sum_{\itersample{1}}\dots \mathbb{E}_{\proposalcond{j+1}} \bigg[ \sum_{\itersample{j+1}}\forward{\nabla_\node^{(n)} \Weight{j+1} \magic(\Multipl{j}) \additive{j+1}} \bigg] \dots \bigg]=0.
    \end{align}
\end{proof} 

    Next, we consider the term $\mathbb{E}_{\proposalcond{k}}\bigg[\sum_{\itersample{k}} \forward { \nabla_\node^{(n)} \Weight{k} \magic(\Multipl{k}) \detnode }\bigg]$ and prove using induction that

    \begin{lemma}
        \label{lemma:multiplicative}
        For any $1\leq j\leq k$, it holds that
    \begin{align}
        \mathbb{E}_{\proposal{1}}\bigg[\sum_{\itersample{1}}
            \dots \mathbb{E}_{\proposalcond{j}}\bigg[\sum_{\itersample{j}} \forward { \nabla_\node^{(n)} \Weight{j} \magic(\Multipl{j}) \detnode' }\bigg]\dots \bigg] = \nabla_\node^{(n)} \mathbb{E}[\detnode]
    \end{align}
    where $\detnode'=\mathbb{E}_{\stochastic_{j+1}, ..., \stochastic_{k}}[\detnode]$. Furthermore, for $1< j \leq k$, it holds that 
    \begin{align}
        & \mathbb{E}_{\proposal{1}}\bigg[\sum_{\itersample{1}}
            \dots \mathbb{E}_{\proposalcond{j}}\bigg[\sum_{\itersample{j}} \forward { \nabla_\node^{(n)} \Weight{j} \magic(\Multipl{j}) \detnode' }\bigg]\dots \bigg] \\
        =& \mathbb{E}_{\proposal{1}}\bigg[\sum_{\itersample{1}}
        \dots \mathbb{E}_{\proposalcond{j-1}}\bigg[\sum_{\itersample{j-1}} \forward { \nabla_\node^{(n)} \Weight{j-1} \magic(\Multipl{j-1}) \mathbb{E}_{\stochastic_j}[\detnode'] }\bigg]\dots \bigg]
    \end{align}
    \end{lemma}
    
    \begin{proof}
        The base case $j=1$ directly follows from condition 1: 
        \begin{align}
            \mathbb{E}_{\proposal{1}}\bigg[\sum_{\itersample{1}} \forward { \nabla_\node^{(n)} \weight{1} \magic(\multipl{i})\detnode'  }\bigg] = \forward{\nabla_\node^{(n)} \mathbb{E}_{\stochastic_1}[\detnode']} = \nabla_\node^{(n)} \mathbb{E}[\detnode],
        \end{align}
        since $\mathbb{E}[\detnode] = \mathbb{E}_{\stochastic_1, ..., \stochastic_k}[\detnode]$ and by the assumption that $\mathbb{E}[\detnode]$ is identical under evaluation. 
    
        Assume the lemma holds for $j<k$ and consider $j+1$. First, we use Proposition \ref{prop:DiCE_multiply} and reorder the terms:
        \begin{align}
            & \mathbb{E}_{\proposalcond{j+1}}\bigg[\sum_{\itersample{j+1}} \forward { \nabla_\node^{(n)} \Weight{j+1}  \magic(\Multipl{j+1}) \detnode' }\bigg] \\
            =& \mathbb{E}_{\proposalcond{j+1}}\bigg[\sum_{\itersample{j}}  \forward {  \nabla_\node^{(n)} \Weight{j+1} \magic(\Multipl{j+1}) \weight{j+1}  \magic( \multipl{j+1})  \detnode' }\bigg] 
        \end{align}
        Next, we again use the general Leibniz rule:
        \begin{align}
            & \mathbb{E}_{\proposalcond{j+1}}\bigg[\sum_{\itersample{j+1}}  \forward {  \nabla_\node^{(n)} \Weight{j} \magic(\Multipl{j}) \weight{j+1}  \magic( \multipl{j+1})  \detnode' }\bigg] \\
            =&\mathbb{E}_{\proposalcond{j+1}}\bigg[\sum_{\itersample{j+1}}  \forward {  \sum_{m=0}^n \binom{n}{m} \nabla_\node^{(n-m)} \Weight{j}   \magic(\Multipl{j}) \nabla_\node^{(m)} \weight{j+1}  \magic( \multipl{j+1})  \detnode' }\bigg] 
        \end{align}
        where we use for the general Leibniz rule $f=\Weight{j}  \magic(\Multipl{j})$ and $g=\weight{j+1}  \magic( \multipl{j+1})  \detnode'$. 
        Note that $\nabla_\node^{(n-m)} \Weight{j} \magic(\Multipl{j})$ does not depend on $\vals{j+1}$. Therefore,
        \begin{align}
            &\mathbb{E}_{\proposalcond{j+1}}\bigg[\sum_{\itersample{j+1}}  \forward {  \sum_{m=0}^n \binom{n}{m} \nabla_\node^{(n-m)} \Weight{j}   \magic(\Multipl{j}) \nabla_\node^{(m)} \weight{j+1}  \magic( \multipl{j+1})  \detnode' }\bigg] \\
            =& \sum_{m=0}^n \binom{n}{m}  \forward {  \nabla_\node^{(n-m)} \Weight{j}   \magic(\Multipl{j})} \mathbb{E}_{\proposalcond{j+1}}\bigg[\sum_{\itersample{j+1}} \forward{ \nabla_\node^{(m)} \weight{j+1}  \magic( \multipl{j+1})  \detnode' }\bigg] \\
            % This line could probably be skipped
            % =& \sum_{m=0}^n \binom{n}{m}  \forward \big(  \nabla_\node^{(n-m)} \Weight{j}   \magic(\Multipl{j})\big) \forward \big( \nabla_\node^{(m)} \mathbb{E}_{\stochastic_{j+1}}[ \detnode' ]\big) \\
            =&   \forward {\sum_{m=0}^n \binom{n}{m}  \nabla_\node^{(n-m)} \Weight{j}   \magic(\Multipl{j}) \nabla_\node^{(m)} \mathbb{E}_{\stochastic_{j+1}}[ \detnode' ]} \\
            =& \forward {  \nabla_\node^{(n)} \Weight{j}   \magic(\Multipl{j}) \mathbb{E}_{\stochastic_{j+1}}[ \detnode' ]}
        \end{align}

        From lines 2 to 3, we use condition 1 to reduce the expectation. In the last line, we use the general Leibniz rule in the other direction. We showed that 

        \begin{align}
            &\mathbb{E}_{\proposal{1}}\bigg[\sum_{\itersample{1}}
            \dots \mathbb{E}_{\proposalcond{j+1}}\bigg[\sum_{\itersample{j+1}} \forward { \nabla_\node^{(n)} \Weight{j+1} \magic(\Multipl{j+1}) \detnode' }\bigg] \dots \bigg] \\
            =& \mathbb{E}_{\proposal{1}}\bigg[\sum_{\itersample{1}}
            \dots \mathbb{E}_{\proposalcond{j}}\bigg[\sum_{\itersample{j}} \forward {  \nabla_\node^{(n)} \Weight{j} \magic(\Multipl{j}) \mathbb{E}_{\stochastic_{j+1}}[ \detnode' ]} \bigg] \dots \bigg] 
            = \nabla_\node^{(n)}\mathbb{E}[\detnode]
        \end{align}
        where we use the inductive hypothesis from step 2 to 3, using that $\mathbb{E}_{\stochastic_{j+1}}[\detnode'] = \mathbb{E}_{\stochastic_{j+1}, ..., \stochastic_k}[\detnode]$. 
    \end{proof}

    Using these two lemmas and condition 4, it is easy to show the theorem:
    \begin{align}
        &\mathbb{E}_{\proposal{1}}\bigg[\sum_{\itersample{1}}
            \dots \mathbb{E}_{\proposalcond{k}}\bigg[\sum_{\itersample{k}} \forward { \nabla_\node^{(n)} \Weight{k} \Big( \Additive{k} +\magic(\Multipl{k}) \detnode \Big) }\bigg]\dots \bigg] \\
        =&\mathbb{E}_{\proposal{1}}\bigg[\sum_{\itersample{1}}
            \dots \mathbb{E}_{\proposalcond{k}}\bigg[\sum_{\itersample{k}} \forward { \nabla_\node^{(n)} \Weight{k} \Additive{k} }\bigg]\dots \bigg] \\
        +&\mathbb{E}_{\proposal{1}}\bigg[\sum_{\itersample{1}}
        \dots \mathbb{E}_{\proposalcond{k}}\bigg[\sum_{\itersample{k}} \forward { \nabla_\node^{(n)} \Weight{k} \magic(\Multipl{k}) \detnode  }\bigg]\dots \bigg] \\
        =& 0 + \nabla_\node^{(n)}\mathbb{E}[\detnode]= \nabla_\node^{(n)}\mathbb{E}[\detnode]
    \end{align}
\end{proof}

Note that we used in the proof that condition 1 implies that $\nabla_\node^{(n)}\mathbb{E}_\proposalcond{i}[\sum_{\itersample{i}} \forward{\weight{i}\magic(\multipl{i})}] = 0$, which can be seen by taking $f(\vals{i}) = 1$ and noting that $\nabla_\node^{(m)}\mathbb{E}_{\stochastic_i}[1] = 0$ for $n> 0$. 
\section{Any-order \tadditive{}}
\label{sec:baselines}
Many gradient estimators are combined with \tadditive{}s to reduce variance.
We consider \tadditive{}s for any-order derivative estimation. 
\cite{maoBaselineAnyOrder2019} introduces an any-order baseline in the context of score functions, but only provides proof that this is the baseline for the second-order gradient estimate. 
We use the \emph{Storchastic} framework to prove that it is also the correct baseline for any-order derivatives\footnote{We use a slight variant of the baseline introduced in \cite{maoBaselineAnyOrder2019} to solve an edge case. We will explain in the end of this section how they differ.}. 
Furthermore, we generalize the ideas behind this baseline to all \tadditive{}s, instead of just score-function baselines.

% TODO: This should probably go somewhere else 
The \tadditive{} that implements any-order baselines is: 
\begin{equation}
    \label{eq:higher-order-baseline}
    \fadditive{i} = (1-\magic(\multipl{i})) \fbaseline{i}.
\end{equation}
First, we show that baselines satisfy condition 2 of Theorem \ref{thrm:storchastic}. 
We will assume here that we take only 1 sample with replacement, but the result generalizes to taking multiple samples in the same way as for the first-order baseline. 
For $n=0$, the any-order baseline evaluates to zero which can be seen by considering $\forward{1-\magic(\multipl{i})}$. 
If $n>0$, then noting that $\baseline{i}$ is independent of $x_i$,
\begin{align}
    \mathbb{E}_{q_i}[ \forward{ \nabla^{(n)}_\node(1-\magic(\multipl{i}))\baseline{i}}]=&   \mathbb{E}_{x_{i}}[\forward{-\baseline{i}\nabla^{(n)}_\node-\magic(\multipl{i})}] 
    % =& \frac{1}{m}\mathbb{E}_{q_i}[\sum_{j=1}^m \forward( -\baseline{i, j}\nabla^{(n)}_\node \magic(\multipl{i, j})\magic(\sum_{j=1}^{i-1} \multipl{j}))] \\
    % =&  \forward( \baseline{i})\mathbb{E}_{x_{i}}[\forward( \sum_{m=0}^n \binom{n}{m} \nabla^{(m)}_\node (1-\magic(\multipl{i}))\nabla^{(n-m)}_\node\magic(\multipl{i-1}))]  \\
    % =& \forward( \baseline{i}) \sum_{m=1}^n \binom{n}{m} \forward(\nabla^{(n-m)}_\node\magic(\Multipl{i-1})) \mathbb{E}_{x_{i}}[\forward(  \nabla^{(m)}_\node (-\magic(\multipl{i})) )]  \\
    = \forward{-\baseline{i}}  \nabla^{(n)}_\node\mathbb{E}_{x_{i}}[1] = 0
\end{align}
We next provide a proof for the validity of this baseline for variance reduction of any-order gradient estimation. To do this, we first prove a new general result on the $\magic$ operator:
\begin{proposition}
    \label{prop:baseline-generator}
    For any sequence of functions $\{\multipl{1}, ..., \multipl{k}\}$, $\magic(\Multipl{k})$ is equivalent under evaluation for orders of differentiation $n>0$ to $\sum_{i=1}^k(\magic(\multipl{i}) - 1)\magic(\Multipl{i-1})$. 
    That is, for all positive numbers $n>0$,
    \begin{equation}
        \forward{\nabla_\node^{(n)} \magic(\Multipl{k})} = \forward{\nabla_\node^{(n)} \sum_{i=1}^k \big(\magic(\multipl{i}) - 1\big) \magic(\Multipl{i-1})}
    \end{equation}
\end{proposition}

\begin{proof}
    We will prove this using induction on $k$, starting with the base case $k=1$. Since $n>0$,
    \begin{equation}
        \forward{\nabla_\node^{(n)}  (\magic(\multipl{1}) - 1) \magic(0)} = \forward{\magic(0)\nabla_\node^{(n)}  \magic(\multipl{1})} = \forward{\nabla_\node^{(n)}  \magic(\multipl{1})}
    \end{equation}
    Next, assume the proposition holds for $k$ and consider $k+1$. Then by splitting up the sum,
    \begin{align}
            \forward{\nabla_\node^{(n)} \sum_{i=1}^{k+1} (\magic(\multipl{i}) - 1) \magic(\Multipl{i-1}) } 
            =&\forward{\nabla_\node^{(n)} (\magic(\multipl{k+1}) - 1) \magic(\Multipl{k}) + \nabla_\node^{(n)} \sum_{i=1}^{k} (\magic(\multipl{i}) - 1) \magic(\Multipl{i-1}) }\\
            =&\forward{\nabla_\node^{(n)} (\magic(\multipl{k+1}) - 1) \magic(\Multipl{k}) + \nabla_\node^{(n)}  \magic(\Multipl{k}) }
    \end{align}
    where in the second step we use the inductive hypothesis.

    We will next consider the first term using the general Leibniz rule:
    \begin{align}
        \forward{\nabla_\node^{(n)} (\magic(\multipl{k+1}) - 1) \magic(\Multipl{k}) }
        = \forward{\sum_{m=0}^n \binom{n}{m} \nabla_\node^{(m)} (\magic(\multipl{k+1}) - 1) \nabla_\node^{(n-m)} \magic(\Multipl{k}) }
    \end{align}
    We note that the term corresponding to $m=0$ can be ignored, as $\forward{\magic(\multipl{k+1}) - 1}=(1-1)=0$. Furthermore, for $m>0$,  $\forward{\nabla_\node^{(m)}(\magic(\multipl{k+1}) - 1)}=\forward{\nabla_\node^{(m)}\magic(\multipl{k+1})}$. Therefore,
    \begin{align}
        \forward{\nabla_\node^{(n)} (\magic(\multipl{k+1}) - 1) \magic(\Multipl{k}) }= \forward{\sum_{m=1}^n \binom{n}{m} \nabla_\node^{(m)} \magic(\multipl{k+1}) \nabla_\node^{(n-m)} \magic(\Multipl{k}) }
    \end{align}
    Finally, we add the other term $\nabla_\node^{(n)}  \magic(\Multipl{k})$ again. Then using the general Leibniz rule in the other direction and Proposition \ref{prop:DiCE_multiply},
    \begin{align}
        =& \forward{\sum_{m=1}^n \binom{n}{m} \nabla_\node^{(m)} \magic(\multipl{k+1}) \nabla_\node^{(n-m)} \magic(\Multipl{k}) + \nabla_\node^{(n)} \magic(\Multipl{k}) } \\
        =& \forward{\sum_{m=1}^n \binom{n}{m} \nabla_\node^{(m)} \magic(\multipl{k+1}) \nabla_\node^{(n-m)} \magic(\Multipl{k}) + \magic(\multipl{k+1})\nabla_\node^{(n)} \magic(\Multipl{k}) } \\
        =& \forward{ \sum_{m=0}^n \binom{n}{m} \nabla_\node^{(m)} \magic(\multipl{k+1}) \nabla_\node^{(n-m)} \magic(\Multipl{k}) }
        = \forward{\nabla_\node^{(n)} \magic(\multipl{k+1})\magic(\Multipl{k})} = \forward{\nabla_\node^{(n)} \magic(\Multipl{k+1})} 
    \end{align}
\end{proof}
Next, we note that we can rewrite the expectation of the \emph{Storchastic} surrogate loss in Equation \eqref{eq:storchastic-expectation} to 
\begin{equation}
   \mathbb{E}_{\proposal{1}} \bigg[ \sum_{\itersample{1}}  + 
        \dots \mathbb{E}_{\proposalcond{k}} \bigg[ \sum_{\itersample{k}}\forward{ \nabla_\node^{(n)} \Weight{k}\Big(\Additive{k} +    
             \magic( \Multipl{k} )  \detnode\Big)  } \bigg] \dots \bigg]
\end{equation}
where $\Additive{k}=\sum_{i=1}^{k} \magic\Big(\sum_{j=1}^{i-1} \multipl{j}\Big) \additive{i}$. This can be seen by using Condition 1 and 4 of Theorem 1 to iteratively move the $\magic\Big(\sum_{j=1}^{i-1} \multipl{j}\Big)\additive{i}$ terms into the expectations, which is allowed since they don't depend on $\stochastic_{>i}$.
\begin{theorem}
    Under the conditions of Theorem 1, 
    \begin{equation}
        \Additive{k} + \magic(\Multipl{k})\detnode \equivforward \sum_{i=1}^k \magic(\Multipl{i-1})(\additive{i} + (\magic(\multipl{i}) - 1)\detnode) + \detnode, 
    \end{equation}
    where $\Additive{k}=\sum_{i=1}^{k} \magic\Big(\sum_{j=1}^{i-1} \multipl{j}\Big) \additive{i}$. 
    % and for $m>0$ to
    % \begin{equation}
    %     \sum_{i=1}^k \magic(\Multipl{i-1})(\additive{i} + \magic(\multipl{i})\detnode) 
    %     % =&\sum_{i=1}^k \magic(\Multipl{i-1})(\additive{i} + \magic(\multipl{i})\detnode)-\magic(\Multipl{i-1})\detnode + \detnode
    % \end{equation}
\end{theorem}
\begin{proof}%
    \begin{align}
        \forward{\nabla_\node^{(n)}(\Additive{k} + \magic(\Multipl{k})\detnode)}
        =& \forward{\nabla_\node^{(n)}\Additive{k} + \sum_{m=0}^n\binom{n}{m}\nabla_\node^{(m)} \magic(\Multipl{k}) \nabla_\node^{(n-m)} \detnode}\label{eq:thrm2-2}\\
        =& \forward{\nabla_\node^{(n)}\Additive{k} + \sum_{m=1}^n\binom{n}{m}\nabla_\node^{(m)} \magic(\Multipl{k}) \nabla_\node^{(n-m)} \detnode + \nabla_\node^{(n)} \detnode}\label{eq:thrm2-3}\\
        =& \forward{\nabla_\node^{(n)}\Additive{k} + \sum_{m=1}^n\binom{n}{m}\nabla_\node^{(m)} \sum_{i=1}^k (\magic(\multipl{i}) - 1)\magic(\Multipl{i-1}) \nabla_\node^{(n-m)} \detnode + \nabla_\node^{(n)} \detnode}\label{eq:thrm2-4} \\
        =& \forward{\nabla_\node^{(n)}\big(\sum_{i=1}^k \magic(\Multipl{i-1}) \additive{i} + \sum_{i=1}^k(\magic(\multipl{i}) - 1) \magic(\Multipl{i-1})\detnode +  \detnode\big)}\label{eq:thrm2-5} \\
        =& \forward{\nabla_\node^{(n)}\big(\sum_{i=1}^k \magic(\Multipl{i-1}) (\additive{i} + (\magic(\multipl{i}) - 1) \detnode) +  \detnode\big)} 
    \end{align}
    From \eqref{eq:thrm2-2} to \eqref{eq:thrm2-3}, we use that $m=0$ evaluates to $\nabla_\node^{(n)} \detnode$. From \eqref{eq:thrm2-3} to \eqref{eq:thrm2-4}, we use Proposition \ref{prop:baseline-generator}. From \eqref{eq:thrm2-4} to \eqref{eq:thrm2-5}, we do a reversed general Leibniz rule on the second term. To be able do that, we use that setting $m=0$ in the second term would evaluate to 0 as $\forward{\magic(\multipl{i}) - 1} = 0$.
\end{proof}
% We say that $\sum_{i=1}^k(\magic(\multipl{i}) - 1)\magic(\sum_{j=1}^{i-1} \multipl{i})=\sum_{i=1}^k \magic(\sum_{j=1}^i\multipl{j}) - \magic(\sum_{j=1}^{i-1}\multipl{j})$, where equality follows from Proposition \ref{prop:DiCE_multiply}, is the \emph{baseline generator} of $\multipl{1}, ..., \multipl{k}$. 

Next, consider the inner computation of the \emph{Storchastic} framework in which all $\additive{i}$ use a baseline of the form in Equation \ref{eq:higher-order-baseline}. Note that $\additive{i}=0$ is also in this form by setting $\baseline{i} = 0$. Assume $n>0$ and without loss of generality\footnote{This is assumed simply to make the notation clearer. If the weights are differentiable, the same thing can be shown using an application of the general Leibniz rule.} assume $\nabla_\node^{(m)}\weight{i}=0$ for all $m$ and $i$. Then using Proposition \ref{prop:baseline-generator},
\begin{align}
    &\prod_{i=1}^k \weight{i} \forward{ \nabla_\node^{(n)} \Big( \sum_{i=1}^k(1-\magic(\multipl{i}))\magic(\Multipl{i-1}) b_i +\magic(\Multipl{k}) \detnode \Big) }\\
    =&\prod_{i=1}^k \weight{i} \forward{ \nabla_\node^{(n)} \Big( -\sum_{i=1}^k(\magic(\multipl{i})-1)\magic(\Multipl{i-1}) b_i +\sum_{i=1}^k(\magic(\multipl{i})-1)\magic(\Multipl{i-1})  \detnode \Big) }\\
    =& \prod_{i=1}^k \weight{i} \forward{ \nabla_\node^{(n)} \sum_{i=1}^k(\magic(\multipl{i})-1)\magic(\Multipl{i-1}) (\detnode - \baseline{i}) }
\end{align}

The intuition behind the variance reduction of this any-order gradient estimate is that all terms of the gradient involving $\multipl{i}$, possibly multiplied with other $\multipl{j}$ such that $j<i$, use the $i$-th baseline $\baseline{i}$. This allows modelling baselines for each sampling step to effectively make use of background knowledge or known statistics of the corresponding set of random variables.

We note that our baseline is slightly different from \cite{maoBaselineAnyOrder2019}, which instead of $\magic(\Multipl{i-1})=\magic(\sum_{j=1}^{i-1}\multipl{j})$ used $\magic(\sum_{\stochastic_{j}\before \stochastic_i} \multipl{j})$.
Although this might initially seem more intuitive, we will show with a small counterexample why we should consider any stochastic nodes ordered topologically before $i$ instead of just those that directly influence $i$. 

Consider the stochastic computation graph with stochastic nodes $p(\stochnode_1|\node)$ and $p(\stochnode_2|\node)$ and cost function $f(x_1, x_2)$. For simplicity, assume we use single-sample score function estimators for each stochastic node. Consider the second-order gradient of the cost function using the recursion in Proposition \ref{prop:DICE}:
\begin{align}
    \nabla_\node^2 \mathbb{E}_{\stochnode_1, \stochnode_2}[f(x_1, x_2)] =& \mathbb{E}_{\stochnode_1, \stochnode_2}[\forward{\nabla_\node^2 \magic(\sum_{i=1}^2\log p(x_i|\node)) f(x_1, x_2)}] \\
    =& \mathbb{E}_{\stochnode_1, \stochnode_2}[f(x_1, x_2)\big(\sum_{i=1}^2 \nabla_\node^2\log p(x_i|\node) + (\nabla_\node \log p(x_i|\node))^2  \\
    &+ 2\nabla_\node\log p(x_1|\node)\nabla_\node\log p(x_2|\node) )) \big)] 
\end{align}
Despite the fact that $x_1$ does not directly influence $x_2$, higher-order derivatives will have terms that involve both the log-probabilities of $x_1$ and $x_2$, in this case $2\nabla_\node\log p(x_1|\node)\nabla_\node\log p(x_2|\node)$. Note that since $a$ does not directly influence $b$, the baseline generated for the second-order derivative using the method in \cite{maoBaselineAnyOrder2019} would be  
\begin{align}
    \forward{\nabla_\node^2 \sum_{i=1}^2\additive{i}} = \forward{\sum_{i=1}^2\nabla_\node^2 (1-\magic(\log p(x_i|\node)))\magic(0) \baseline{i} }
    = -\sum_{i=1}^2\nabla_\node^2 \log p(x_i|\node))) \baseline{i} 
\end{align}
This baseline does not have a term for $2\nabla_\node\log p(x_1|\node)\nabla_\node\log p(x_2|\node)$, meaning the variance of that term will not be reduced through a baseline. The baseline introduced in Equation \ref{eq:higher-order-baseline} will include it, since
\begin{align}
    \forward{\nabla_\node^2\sum_{i=1}^2\additive{i}} =& \forward{\nabla_\node^2 (1-\magic(\log p(x_1|\node))) \baseline{i} + (1-\magic(\log p(x_2|\node)))\magic(\log p(x_1|\node)) \baseline{i} }\\
    =& -\sum_{i=1}^2\nabla_\node^2 \log p(x_i|\node))) \baseline{i} - 2\nabla_\node\log p(x_1|\node)\nabla_\node\log p(x_2|\node)
\end{align}

Designing a good baseline function $\fbaseline{i, j}$ that will reduce variance significantly is highly application dependent. Simple options are a moving average and the leave-one-out baseline, which is given by $\fbaseline{i} = \frac{1}{m-1}\sum_{j'=1, j'\neq j}^m \bot( \circ )f(\vals{< i}, x_{i, j'}))$ \cite{koolBuyREINFORCESamples2019,mnihVariationalInferenceMonte2016}. More advanced baselines can take into account the previous stochastic nodes $\stochastic_1, ..., \stochastic_{i-1}$ \cite{weberCreditAssignmentTechniques2019}. Here, one should only consider the stochastic nodes that directly influence $\stochnode_i$, that is, $\stochastic_{\before i}$. Another popular choice is self-critical baselines \cite{rennieSelfCriticalSequenceTraining2017,koolAttentionLearnSolve2019} that use deterministic test-time decoding algorithms to find $\hat{\vals{i}}$ and then evaluate it, giving $\fbaseline{i}=f(\hat{\vals{i}})$.

\section{Examples of Gradient Estimators}
\label{sec:gradient-estimators}
In this section, we prove the validity of several gradient estimators within the \emph{Storchastic} framework, focusing primarily on discrete gradient estimation methods. 

% First, we rewrite condition 1 of the theorem to get a better insight into how to prove the condition:

% \begin{align}
%     \nabla_\node^{(m)} \mathbb{E}_{\stochastic_i}[f(\vals{i})] = \nabla_\node^{(m)} \int_{\sspace_i} p(\vals{i}) f(\vals{i}) d\vals{i} 
%     = \int_{\sspace_i} \nabla_\node^{(m)} (p(\vals{i}) f(\vals{i}) d\vals{i})
% \end{align}
\subsection{Expectation}
\label{sec:expectation}
Assume $p(x_i)$ is a discrete (ie, categorical) distribution with a finite amount of classes $1, ..., C_i$. While this is not an estimate but the true gradient, it fits in the \emph{Storchastic} framework as follows:

\begin{enumerate}
    \item $\weight{i}(x_i)=p(x_i|\vals{<i})$
    \item $\fproposalcond{i} = \delta_{\{1, ..., C_i\}}(\sampleset_i)$ (that is, a dirac delta distribution with full mass on sampling exactly the sequence $\{1, ..., C_i\}$)
    \item $\multipl{i}(x_i)=0$
    \item $\additive{i}(x_i)=0$
\end{enumerate}

Next, we prove the individual conditions to show that this method can be used within \emph{Storchastic}, starting with condition 1:
\begin{align}
    \mathbb{E}_{\proposalcond{i}}[\sum_{x_i\in \sampleset{i}} \forward{\nabla_\node^{(n)} \weight{i}\magic(\multipl{i})f(x_i)}] 
    =& \sum_{j=1}^{C_i} \forward{\nabla_\node^{(n)} p(x_i=j|\vals{<i})\magic(0)f(j)} \\
    =& \forward{\sum_{j=1}^{C_i} \sum_{m=0}^n  \nabla_\node^{(n-m)} p(x_i=j|\vals{<i})\nabla_\node^{(m)}\magic(0)f(j)}
\end{align}
Using the recursion in Proposition \ref{prop:DICE}, we see that $\forward{\nabla_\node^{(m)}\magic(0)f(j)}=\nabla_\node^{(m)} f(j)$, since $\nabla_\node \multipl{i}=\nabla_\node 0= 0$. So,
\begin{align}
    \forward{\sum_{j=1}^{C_i} \sum_{m=0}^n  \nabla_\node^{(n-m)} p(x_i=j|\vals{<i})\nabla_\node^{(m)} f(j)} = \forward{\sum_{j=1}^{C_i} \nabla_\node^{(n)} p(x_i=j|\vals{<i}) f(j)} = \nabla_\node^{(n)}\mathbb{E}_{x_i}[f(x_i)].
\end{align}

% where the last step follows from the assumption of equivalence under evaluation of $\nabla_\node^{(n)}\mathbb{E}_{x_i}[f(x_i)]$. 

Condition 2 follows simply from $\additive{i}(x_i)=0$, and condition 3 follows from the fact that $\sum_{j=1}^{C_i}p(x_i=j|\vals{<i}) = 1$, that is, constant. Condition 4 follows from the SCG being identical under evaluation, ie $\forward{p(x_i=j|\vals{<i})}=p(x_i=j|\vals{<i})$. 

It should be noted that this proof is not completely trivial, as it shows how to implement the expectation so that it can be combined with other gradient estimators while making sure the pathwise derivative through $f$ also gets the correct gradient.  

% TODO: This isn't worked out... What paper was this?
% \subsubsection{Local Expectation Gradients}
% Since deriving the full expectation over a set of discrete random variables is computationally exponential in the amount of variables, a more efficient alternative is to take local expectations

\subsection{Score Function}
The score function is the best known general gradient estimator and is easy to fit in \emph{Storchastic}. 

\subsubsection{Score Function with Replacement}
\label{sec:sfwr}
We consider the case where we take $m$ samples with replacement from the distribution $p(x_i|\vals{<i})$, and we use a baseline $\fbaseline{i}$ for the first-order gradient estimate. 
\begin{enumerate}
    \item $\weight{i}(x_i)= \frac{1}{m}$
    \item $\proposalcond{i}=\prod_{j=1}^m p(x_{i, j}|\vals{<i})$. That is, $x_{i, 1}, ..., x_{i, m}\sim p(x_i|\vals{<i})$.
    \item $\multipl{i}(x_i)=\log p(x_i|\vals{<i})$
    \item $\fadditive{i} = (1-\magic(\multipl{i})) \fbaseline{i}$, where $\fbaseline{i}$ is not differentiable, that is, $\forward{\nabla_\node^{(n)}\fbaseline{i}}=0$ for $n>0$. 
\end{enumerate}
We start by showing that condition 1 holds. We assume $p(x_i|\vals{<i})$ is a continuous distribution and note that the proof for discrete distributions is analogous. 

We will show how to prove that sampling a set of $m$ samples with replacement can be reduced in expectation to sampling a single sample. Here, we use that $x_{i,1}, ..., x_{i,m}$ are all independently (line 1 to 2) and identically (line 2 to 3) distributed.
\begin{align}
    \mathbb{E}_{\proposalcond{i}}[\sum_{j=1}^m \forward{\nabla_\node^{(n)} \frac{1}{m}\magic(\multipl{i, j}) f(\vals{< i}, x_{i, j}) }] 
    =&\frac{1}{m} \sum_{j=1}^m \mathbb{E}_{x_{i,j}\sim p(x_i)}[ \forward{\nabla_\node^{(n)}  \magic(\multipl{i, j}) f(\vals{\leq i}, x_{i, j}) } ]\\
    =&\frac{1}{m} \sum_{j=1}^m \mathbb{E}_{x_{i}\sim p(x_i)}[ \forward{\nabla_\node^{(n)}  \magic(\multipl{i}) f(\vals{\leq i}) } ] =\mathbb{E}_{x_i}\Big[\forward{\nabla_\node^{(n)}  \magic(\multipl{i}) f(\vals{\leq i}) } \Big]
\end{align}
% \begin{align}
%     &\mathbb{E}_{\proposalcond{i}}[\sum_{j=1}^m \forward\big(\nabla_\node^{(n)}) \frac{1}{m}\magic(\multipl{i, j}) f(\vals{\leq i}) \big)] \\
%     =&\frac{1}{m} \int_{\sspace_i^m} \prod_{j=1}^m p(x_{i, j} | \vals{<i}) \sum_{j=1}^m \forward\big(\nabla_\node^{(n)})  \magic(\multipl{i, j}) f(\vals{\leq i}) \big) d\sampleset_i \\
%     =&\frac{1}{m} \sum_{j=1}^m \int_{\sspace_i^m} p(x_{i, m}|\vals{<i})  \forward\big(\nabla_\node^{(n)}  \magic(\multipl{i, m}) f(\vals{\leq i}) \big) \prod_{j'=1}^{m-1} p(x_{i, j'}| \vals{<i})  d\sampleset_i \\
%     =& \int_{\sspace_i} p(x_{i, m}|\vals{<i})  \forward\big(\nabla_\node^{(n)}  \magic(\multipl{i, m}) f(\vals{\leq i}) \big) \\
%     &\int_{\sspace_i^{m-1}}\prod_{j=1}^{m-1} p(x_{i, j}| \vals{<i})  d(x_{i, 1}, ..., x_{i, m-1}) dx_{i, m} \\
%     =& \int_{\sspace_i} p(x_{i, m}|\vals{<i})  \forward\big(\nabla_\node^{(n)}  \magic(\multipl{i,m}) f(\vals{\leq i}) \big) dx_{i, m} \\
%     =&\mathbb{E}_{x_i}\Big[\forward\big(\nabla_\node^{(n)}  \magic(\multipl{i}) f(\vals{\leq i}) \big) \Big]
% \end{align}
A proof that $\mathbb{E}_{x_i}\Big[\forward{\nabla_\node^{(n)}  \magic(\multipl{i}) f(\vals{\leq i}) } \Big] =  \nabla_\node^{(n)} \mathbb{E}_{x_i}[f(\vals{\leq i})]$ was first given in \cite{foersterDiCEInfinitelyDifferentiable2018}. For completeness, we give a similar proof here, using induction.

First, assume $n=0$. Then, $\mathbb{E}_{x_i}[\forward{\magic(\multipl{i})}f(\vals{\leq i})] = \mathbb{E}_{x_i}[\forward{f(\vals{\leq i})}] = \mathbb{E}_{x_i}[f(\vals{\leq i}]$.

Next, assume it holds for $n$, and consider $n+1$. Using Proposition \ref{prop:DICE}, we find that $g^{(n+1)}(\vals{\leq i}) = \nabla_\node g^{(n)}(\vals{\leq i}) + g^{(n)}(\vals{\leq i})\nabla_\node \log p(x_i|\vals{<i})$. Writing the expectation out, we find 
\begin{align}
    &\mathbb{E}_{x_i}[\forward{\nabla_\node g^{(n)}(\vals{\leq i}) + g^{(n)}(\vals{\leq i})\nabla_\node \log p(x_i|\vals{<i})}] \\
    =&\int \forward{p(x_i|\vals{<i})(\nabla_\node g^{(n)}(\vals{\leq i}) + g^{(n)}(\vals{\leq i}) \frac{\nabla_\node p(x_i|\vals{<i})}{p(x_i|\vals{<i})})} dx_i\\ 
    % =&\int \forward(p(x_i|\vals{<i})\nabla_\node g^{(n)}(x_i) + g^{(n)}(x_i) \nabla_\node p(x_i|\vals{<i}))) dx_i \\
    =& \int \forward{\nabla_\node p(x_i|\vals{<i})g^{(n)}(\vals{\leq i})} dx_i 
    =\forward{\nabla_\node\mathbb{E}_{x_i}[   g^{(n)}(\vals{\leq i})]} 
\end{align}
By Proposition~\ref{prop:equiv-eval}, $g^{(n)}(x_i)$ is identical under evaluation, since by the assumption of Theorem~\ref{thrm:storchastic} both $p(x_i|\vals{<i})$ and $f(\vals{\leq i})$ are identical under evaluation. As a result, $\forward{\nabla_\node \mathbb{E}_{x_i}[ g^{(n)}(\vals{\leq i}) ]} = \nabla_\node \mathbb{E}_{x_i}[  \forward {g^{(n)}(\vals{\leq i})}]$. Therefore, by the induction hypothesis,
\begin{align}
    \mathbb{E}_{\proposalcond{i}}[\forward{g^{(n+1)}(\vals{\leq i})}]=\nabla_\node \mathbb{E}_{x_i}[  \forward {g^{(n)}(\vals{\leq i})}] = \nabla_\node^{(n+1)} \mathbb{E}_{x_i}[ f(\vals{\leq i})]
\end{align}

Since the weights ($\frac{1}{m}$) are constant, condition 3 is satisfied.

\subsubsection{Importance Sampling}
\label{sec:importance-sampling}
A common use case for weighting samples is importance sampling \cite{rubinsteinSimulationMonteCarlo2016}. In the context of gradient estimation, it is often used in off-policy reinforcement-learning \cite{mahmoodWeightedImportanceSampling2014} to allow unbiased gradient estimates using samples from another policy. For simplicity, we consider importance samples within the context of score function estimators, single-sample estimates, and use no baselines. The last two can be introduced using the techniques in Section \ref{sec:sfwr} and \ref{sec:baselines}.
\begin{enumerate}
    \item $\weight{i} = \bot(\frac{p(x_i|\vals{<i})}{q(x_i|\vals{<i})})$,
    \item $q(x_i|\vals{<i})$ is the sampling distribution,
    \item $\fmultipl{i} = \log p(x_i|\vals{<i})$,
    \item $\fadditive{i} = 0$.
\end{enumerate}
Condition 3 follows from the fact that $\nabla_\node^{(n)}\weight{i}=0$ for $n>0$, since the importance weights are detached from the computation graph. Condition 1:
\begin{align}
    \mathbb{E}_{\proposalcond{i}}[\forward{\nabla_\node^{(n)} \bot\Big(\frac{p(x_i|\vals{<i})}{q(x_i|\vals{<i})}\Big) \magic(\multipl{i}) f(\vals{\leq i}) }] 
    =&\int_{\sspace_i}q(x_i|\vals{<i}) \frac{p(x_i|\vals{<i})}{q(x_i|\vals{<i})} \forward{\nabla_\node^{(n)}  \magic(\multipl{i}) f(\vals{\leq i}) } dx_i\\
    =&\mathbb{E}_{\stochastic_i}[\forward{\nabla_\node^{(n)} \magic(\multipl{i}) f(\vals{\leq i}) }] = \forward{\mathbb{E}_{\stochastic_i}[f(\vals{\leq i})]}
\end{align}
where in the last step we use the proven condition 1 of \ref{sec:sfwr}. Note that this holds both for $n=0$ and $n>0$.

\subsubsection{Discrete Sequence Estimators}
\label{sec:discrete-seqs}
Recent literature introduced several estimators for sequences of discrete random variables. These are quite similar in how they are implemented in \emph{Storchastic}, which is why we group them together.

The sum-and-sample estimator chooses a set of sequences $\hat{\sampleset_i}\subset \sspace_{i}$ and chooses $k - |\hat{\sampleset_i}|>0$ samples from $\sspace_i \setminus \hat{\sampleset_{i}}$. This set can be the most probable sequences \cite{liuRaoBlackwellizedStochasticGradients2019} or can be chosen randomly \cite{koolEstimatingGradientsDiscrete2020}. This is guaranteed not to increase variance through Rao-Blackwellization \cite{casellaRaoblackwellisationSamplingSchemes1996,liuRaoBlackwellizedStochasticGradients2019}. It is often used together with deterministic cost functions $f$, which allows memorizing the cost-function evaluations of the sequences in $\hat{\sampleset_i}$. In this context, the estimator is known as Memory-Augmented Policy Optimization \cite{liangMemoryAugmentedPolicy2019}. 
\begin{enumerate}
    \item $\fweight{i}=I[\vals{i}\in \hat{\sampleset_i}] p(\vals{i}|\vals{<i}) + I[\vals{i}\not\in \hat{\sampleset_i}] \frac{p(\vals{i}\not\in \hat{\sampleset_i})}{k - |\hat{\sampleset_i}|}$
    \item $q(\sampleset_{i}) = \delta_{\hat{\sampleset_i}}(\vals{i, 1}, ..., \vals{i, |\hat{\sampleset_i}|})\cdot \prod_{j=|\hat{\sampleset_i}| + 1}^{k} p(\vals{i, j}|\vals{i, j}\not\in \hat{\sampleset_{i}}, \vals{<i})$
\end{enumerate}
were $p(\vals{i}\not\in \hat{\sampleset_i})=1-\sum_{\vals{i}'\in \hat{\sampleset_i}  }p(\vals{i}'|\vals{<i})$. This essentially always `samples' the set $\hat{\sampleset_i}$ using the Dirac delta distribution, and then samples $k$ more samples out of the remaining sequences, with replacement. The  estimator resulting from this implementation is
\begin{align}
    \mathbb{E}_{\proposalcond{i}}[\sum_{j=1}^{|\hat{\sampleset_i}|} p(\vals{i, j}|\vals{<i}) f(\vals{<i}, \vals{i, j}) + \sum_{j=|\hat{\sampleset_i}|+1}^k \frac{p(\vals{i}\not\in \hat{\sampleset_i})}{k - |\hat{\sampleset_i}|} \magic(\multipl{i})f(\vals{<i}, \vals{i, j})]
\end{align}
Using the result from Section \ref{sec:expectation}, we see that 
\begin{align}
\forward{\nabla_\node^{(n)}  \mathbb{E}_{\proposalcond{i}}[\sum_{j=1}^{|\hat{\sampleset_i}|} p(\vals{i, j}|\vals{<i}) f(\vals{<i}, \vals{i, j})]} 
=&\forward{\nabla_\node^{(n)} p(\vals{i}\in \hat{\sampleset_i}) \mathbb{E}_{\proposalcond{i}}[\sum_{j=1}^{|\hat{\sampleset_i}|} p(\vals{i, j}|\vals{i, j}\in \hat{\sampleset_i}, \vals{<i}) f(\vals{<i}, \vals{i, j})]} \\
=& \nabla_\node^{(n)} p(\vals{i}\in \hat{\sampleset_i})\mathbb{E}_{p(\vals{i}|\vals{i}\in \hat{\sampleset_{i}}, \vals{<i})}[f(\vals{\leq i} )].
\end{align}
Similarly, from the result for sampling with replacement of score functions in Section \ref{sec:sfwr}, 
\begin{align}
    &\forward{\nabla_\node^{n}\mathbb{E}_{\proposalcond{i}}[\sum_{j=|\hat{\sampleset_i}|+1}^k \frac{p(\vals{i}\not\in \hat{\sampleset_i})}{k - |\hat{\sampleset_i}|} \magic(\multipl{i})f(\vals{<i}, \vals{i, j})]} 
    =\nabla_\node^{(n)} p(\vals{i}\not\in \hat{\sampleset_i})\mathbb{E}_{p(\vals{i}|\vals{i}\not\in \hat{\sampleset_{i}}, \vals{<i})}[f(\vals{\leq i} )]
\end{align}
Added together, these form $\nabla_\node^{(n)}\mathbb{E}_{\stochastic_i}[f(\vals{\leq i})]$, which shows that the sum-and-sample estimator with the score function is unbiased for any-order gradient estimation. The variance of this estimator can be further reduced using a baseline from Section \ref{sec:baselines}, such as the leave-one-out baseline.

The \emph{unordered set estimator} is a low-variance gradient estimation method for a sequence of discrete random variables $\stochastic_i$ \cite{koolEstimatingGradientsDiscrete2020}. It makes use of samples without replacement to ensure that each sequence in the sampled batch will be different. We show here how to implement this estimator within \emph{Storchastic}, leaving the proof for validity of the estimator for \cite{koolEstimatingGradientsDiscrete2020}. 

\begin{enumerate}
    \item $\fproposalcond{i}$ is an ordered sample without replacement from $p(\stochastic_i|\vals{<i})$. For sequences, samples can efficiently be taken in parallel using ancestral gumbel-top-k sampling \cite{koolAncestralGumbeltopkSampling2020,koolStochasticBeamsWhere2019}. An ordered sample without replacement means that we take a sequence of samples, where the $i$th sample cannot equal the $i-1$ samples before it.
    \item $\fweight{i} = \bot\Big( \frac{p(\vals{i}|\vals{<i}) p(U=\sampleset_i|o_1=\vals{i}, \vals{<i})}{p(U=\sampleset_i|\vals{<i})} \Big)$, where $p(U=\sampleset_i|\vals{<i})$ is the probability of the \emph{unorderd} sample without replacement, and $p(U=\sampleset_i|o_1=\vals{i}, \vals{<i})$ is the probability of the unordered sample without replacement, given that, if we were to order the sample, the first of those ordered samples is $\vals{i}$.
    \item $\fmultipl{i} = \log p(\vals{i}|\vals{<i})$
    \item $\fadditive{i} =  (1-\magic(\multipl{i})) \baseline{i}(\vals{<i}, \sampleset_i)$, where~$\baseline{i}(\vals{<i}, \sampleset_i)= \\ \sum_{\vals{i}'\in\sampleset_i}\bot\Big( \frac{p(\vals{i}'|\vals{<i})p(U=\sampleset_i|o_1=\vals{i}, o_2=\vals{i}', \vals{<i})}{p(U=\sampleset_i|o_1=\vals{i}, \vals{<i})} f(\vals{i}') \Big)$
\end{enumerate}

This estimator essentially reweights each sample without replacement to ensure it remains unbiased under this sampling strategy. This estimator can be used for any-order differentiation, since $\mathbb{E}_{\proposalcond{i}}[\sum_{\itersample{i}}\forward{\weight{i} f(\vals{i})}] = \mathbb{E}_{\stochastic_i}[\forward{f(\vals{i})}]$ (see \cite{koolEstimatingGradientsDiscrete2020} for the proof) and $\forward{\nabla_\node^{(n)} \weight{i}}=0$ for $n>0$. The baseline is 0 in expectation for the zeroth and first order evaluation \cite{koolEstimatingGradientsDiscrete2020}. We leave for future work whether it is also a mean-zero baseline for $n>1$. 

% Unordered set estimator: $q$ is the distribution over \textit{unordered} samples without replacement,
% in this case our set $\sample{i}$.
% The \textit{unordered} sample $\sample{i}$ is generated
% from an \textit{ordered} sample without replacement $B_{\stochastic_i}$ by simply removing the ordering.
% This can be done using rejection sampling or the gumbel-top-k trick, which is extended to ancestral sampling in
% ancestral gumbel-top-k. $\weight{i}=\frac{p(\vals{i}|\vals{<i})
%  q(\sample{i} |b_1=\vals{i}, \vals{<i}) }
% {q(\sample{i}| \vals{<i})}$, where
% $q(\sample{i} |b_1=\vals{i}, \vals{<i})$ is the probability
% of the \textit{unordered} SWOR $\sample{i}$, given that the first element of the \textit{ordered} sample
% $b_1$ is $\value{i}$.
% Using Bayes theorem, $\weight{i}=q(b_1=\vals{i}|\sample{i}, \vals{<i})$,
% or the probability that the first element of the \textit{ordered} SWOR that generated $\sample{i}$ takes the value $\vals{i}$.
% In other words, we compute the following: $\mathbb{E}_{q(\sample{i} | \vals{<i}) } \left[ \mathbb{E}_{q(b_1|\sample{i})}[f(b_1)] \right] =
% \mathbb{E}_{q(\sampleset_\stochnode | \bx_{\stochastic_{<i}}) } \left[ \mathbb{E}_{q(B_{\stochastic_i} | \sample{i})}[f(b_1)] \right]=\mathbb{E}_{q(B_{\stochastic_i} | \stochastic_{<i})}[f(b_1)]
% =\mathbb{E}_{\stochastic_i | \stochastic_{<i}}[f(\bx_\stochnode)] $.
% The last step follows because in samples without replacement, the first sample is the same as a normal sample.
% See the paper for details on computing $w_i$.

\subsubsection{LAX, RELAX and REBAR}
\label{sec:relax}
REBAR \cite{tuckerREBARLowvarianceUnbiased2017} and LAX and RELAX \cite{grathwohlBackpropagationVoidOptimizing2018} are single-sample score-function based methods that learn a control variate to minimize variance. 
The control variate is implemented using reparameterization. 
We start with LAX as it is simplest, and then extend the argument to RELAX, since REBAR is a special case of RELAX.
We use $\baseline{i, \phi}$ to denote the learnable control variate. 
We have to assume there is no pathwise dependency of $\node$ with respect to $\baseline{i, \phi}$. Furthermore, we assume $\vals{i}$ is a reparameterized sample of $p(\vals{i}|\vals{\leq i})$.
The \tadditive{} component then is: 
\begin{equation}
    \fadditive{i} = \baseline{i, \phi}(\vals{\leq i}) - \magic(\multipl{i})\bot(\baseline{i, \phi}(\vals{\leq i}))
\end{equation}
Since LAX uses normal single-sample score-function, we only have to show condition 2, namely that this \tadditive{} component has 0 expectation for all orders of differentiation. 
\begin{align}
    \mathbb{E}_{\stochastic_i}\bigg[\forward{\nabla_\node^{(n)} \big(\baseline{i, \phi} - \magic(\multipl{i})\bot(\baseline{i, \phi}) \big) }\bigg] = 0
    % =&\forward{\sum_{m=0}^n \nabla_\node^{n-m} \magic(\sum_{j=1}^{i-1}\multipl{j}) \mathbb{E}_{\stochastic_i}\Big[ \nabla_\node^{(m)} (\baseline{i, \phi} - \magic(\multipl{i})\bot(\baseline{i, \phi}))\Big]}
    % =& \mathbb{E}_{\stochastic_i}\bigg[\forward\Big(\nabla_\node^{(n)} \magic(\sum_{j=1}^{i-1}\multipl{j})\Big((1-\magic(\multipl{i})) \bot(\baseline{i, \phi}) + \baseline{i, \phi} - \bot(\baseline{i, \phi} )\Big)\Big)\bigg]
\end{align}
$\forward{\mathbb{E}_{\stochastic_i}[\nabla_\node^{(m)} \baseline{i, \phi}]}$ is the reparameterization estimate of $\forward{\nabla_\node^{(m)} \mathbb{E}_{\stochastic_i}[\baseline{i, \phi}]}$ and $\forward{\mathbb{E}_{\stochastic_i}[\nabla_\node^{(m)} \magic(\log p(\vals{i}|\vals{\leq i}))\bot(\baseline{i, \phi})]}$ is the score-function estimate under the assumption that $\baseline{i, \phi}$ has no pathwise dependency. 
As both are unbiased expectations of the $m$-th order derivative, their difference has to be 0 in expectation, proving condition 2. 
Furthermore, the 0th order evaluation is exactly 0.  
% We multiply the original control variate of LAX by $\magic(\sum_{j=1}^{i-1} \multipl{j})$ to ensure it also reduces the variance of higher-order derivatives (see Appendix \ref{sec:baselines}). 
The parameters $\phi$ are trained to minimize the gradient estimate variance.

The \tadditive{} for RELAX \cite{grathwohlBackpropagationVoidOptimizing2018}, an extension of LAX to discrete random variables, is similar. 
It first samples a continuously relaxed input $q(z_i|\vals{<i})$, which is then transformed to a discrete sample $\vals{i}\sim p(\vals{i}|\vals{<i})$. 
See \cite{grathwohlBackpropagationVoidOptimizing2018,tuckerREBARLowvarianceUnbiased2017} for details on how this relaxed sampling works.
It also samples a relaxed input \emph{condition on the discrete sample}, ie $q(\tilde{z_i}|\vals{\leq i})$.
The corresponding \tadditive{} is
\begin{equation}
    \fadditive{i} = \baseline{i, \phi}(z_i) - \bot(\baseline{i, \phi}(z_i) ) - \baseline{i, \phi}(\tilde{z_i}) + (2 - \magic(\multipl{i}))\bot(\baseline{i, \phi}(\tilde{z_i}))
\end{equation}
Here, we subtract $\bot(\baseline{i, \phi}(z_i))$ to ensure the first two terms together sum to 0 during 0th order evaluation, and add $2 \bot(\baseline{i, \phi}(\tilde{z_i}))$ to ensure the last two terms sum to 0.  
Note that for $n>0$, $\forward{\nabla_\node^{(n)} \fadditive{i}}=\forward{\nabla_\node^{(n)} \big(\baseline{i, \phi}(z_i) - \baseline{i, \phi}(\tilde{z_i}) - \magic(\multipl{i}) \bot(\baseline{i, \phi}(\tilde{z_i}))}\big)$. 
We refer the reader to \cite{grathwohlBackpropagationVoidOptimizing2018, tuckerREBARLowvarianceUnbiased2017} for details on why this \tadditive{} is zero in expectation for 1st order differentiation.
We note that the results extend to higher-order differentiation since the $n$-th order derivative of $\magic(\multipl{i})$ gives $n$th-order score functions which are unbiased expectations of the $n$-th order derivative.

\subsubsection{ARM}
ARM is a score-function based estimator for multivariate Bernouilli random variables. 
For our implementation, we use the baseline formulation mentioned in \cite{yinARMAugmentREINFORCEMergeGradient2019}, and we follow the derivation in terms of the Logistic random variables from \cite{dongDisARMAntitheticGradient2020}. 
ARM assumes a real-valued parameter vector $\alpha$, which can be the output of a neural network. 
The probabilities of the Bernoulli random variable are then assumed to be $\sigma(\alpha)$ where $\sigma$ is the sigmoid function. 

\begin{enumerate}
    \item $\fproposalcond{i}$ is a reparameterized sample from the multivariate Bernouilli distribution. 
    First, it samples $\beps\sim \operatorname{Logistic}(\bzero, \boldsymbol{1})$. 
    Define $\bz_i=\alpha + \beps$ and $\tilde{\bz_i}=\alpha - \beps$.
    We find $\vals{i} = I[\bz_i > \bzero]$. Then, with this procedure, $\vals{i}\sim \operatorname{Bernouilli}(\sigma(\alpha))$. 
    \item $\fweight{i} = 1$
    \item $\fmultipl{i} = \log q_\alpha(\bz_i)$, where $q_\alpha$ is the density function of $\operatorname{Logistic}(\alpha, 1)$. 
    \item $\fadditive{i} = \magic(1-\fmultipl{i}) \frac{1}{2} (f(\vals{<i}, \bz_i > \bzero) + f(\vals{<i}, \tilde{\bz_i} > \bzero))$
\end{enumerate}

Since $\mathbb{E}_{\vals{i}\sim\operatorname{Bernoulli}(\sigma(\alpha_\btheta))}[f(\vals{i})] = \mathbb{E}_{\beps\sim \operatorname{Logistic}(0, 1)}[f(\alpha_\btheta + \beps > \bzero)] = \mathbb{E}_{\bz_i\sim \operatorname{Logistic}(\alpha_\btheta, 1)}[f(\bz_i > \bzero)]$, any unbiased estimate of the logistic reparameterization must also be an unbiased estimate of the original Bernouilli formulation. 
This equality follows because the CDF of the logistic distribution is the logistic function (that is, the sigmoid function). 
$\fmultipl{i}$ is the (unbiased) score function of the logistic reparameterization, which we proved to be an unbiased estimate.

The \tadditive{} has expectation 0 for zeroth and first order differentiation. 
This is because it relies on the score function being an odd function \cite{buesingStochasticGradientEstimation2016}, that is, $\nabla_\node \log q_{\alpha}(\bz_i) = - \nabla_\node \log q_{\alpha}(\tilde{\bz_i})$.
Therefore, $\mathbb{E}_{\beps}[(f(\vals{<i}, \bz_i > 0) + f(\vals{<i}, \tilde{\bz_i} > 0)) \nabla_\node \log q_\alpha(\bz_i)] = \mathbb{E}_{\beps}[f(\vals{<i}, \bz_i > 0) \nabla_\node \log q_\alpha(\bz_i) - f(\vals{<i}, \tilde{\bz_i} > 0) \nabla_\node \log q_\alpha(\tilde{\bz_i})]$. 
Note that, by symmetry of the logistic distribution, $\mathbb{E}_{\beps}[f(\vals{<i}, \bz_i > 0) \nabla_\node \log q_\alpha(\bz_i)] = - \mathbb{E}_{\beps}[f(\vals{<i}, \tilde{\bz_i} > 0) \nabla_\node \log q_\alpha(\tilde{\bz_i})]$, meaning the baseline is zero in expectation.
However, this derivation only holds for odd functions! 
Unfortunately, the second-order score function $\frac{\nabla_\node^{(2)} q_\alpha(\bz_i)}{q_\alpha(\bz_i)}$ is an even function since the derivative of an odd function is always an even function. 
Therefore, the ARM estimator will only be unbiased for first-order gradient estimation.

% \begin{itemize}
%     \item ARM (multivariate bernoulli with $\btheta = \log \mu_\stochnode(\bx_\stochnode=\bone|\bx_{\pa(\stochnode)})$): Let $\beps \sim\operatorname{Logistic}(\bzero, \bone)$ so that $\bz=\beps +\btheta \sim q_\btheta(\bz)=\operatorname{Logistic}(\btheta, \bone)$, then let $\bx_\stochnode =  \boldsymbol{I}[\bz > 0]\sim p$.
%     Weighting is just 1.
%     Mutiplicative term: $\multipl{} = \log q(\bz)$, baseline: $\baseline{}=f_\detnode(\boldsymbol{I}[-\beps + \btheta > \bzero])$.
%     Can be implemented by running $f$ for $\beps + \btheta$ and $-\beps + \btheta$ in parallel.
% \end{itemize}

% TODO
% Okay so this one is hard because for higher-orders, the baseline won't apply to the interacting terms. We would kinda need something to equal $\mathbb{E}[(1-\magic(\multipl{i}))\magic(\sum_{j=1}^{i-1}\multipl{j}) c_\phi(\vals{\leq i})]$. Like the reparameterization gradient of $c_\phi(\vals{\leq i})$ doesn't seem to do this in expectation, right? Probably just equal to to without the $\magic(\sum_{j=1}^{i-1})$. Can we add this somehow?

\subsubsection{GO Gradient}
\label{sec:gogradient}
The GO gradient estimator \cite{congGOGradientExpectationbased2019} is a method that uses the CDF of the distribution to derive the gradient. 
For continuous distributions, it reduces to implicit reparameterization gradients which can be implemented through transforming the computation graph, like other reparameterization methods.
For $m$ independent discrete distributions of $d$ categories, the first-order gradient is given as:
\begin{equation}
    \mathbb{E}_{p(\vals{i}|\vals{\leq i})}\Big[\sum_{j=1}^m (f(\vals{\leq i}) - f(\vals{\leq i\setminus \vals{i, j}}, \vals{i, j} + 1)) \frac{\nabla_\node\sum_{k=1}^{\vals{i, j}}  p_j(k|\vals{<i})}{p_j(\vals{i, j}|\vals{<i})} \Big]
\end{equation}
Note that if $\vals{i, j}=d$, then the estimator evaluates to zero since $\nabla_\node \sum_{k=1}^{d} p_j(k|x_{<i})=0$.

We derive the \emph{Storchastic} implementation by treating the GO estimator as a \tadditive{} of the single-sample score function.
To find this \tadditive{}, we subtract the score function from this estimator,
that is, we subtract $f(\vals{\leq i})\nabla_\node\log p(\vals{i}|\vals{< i}) = f(\vals{\leq i}) \sum_{j=1}^m \nabla_\node \log p(\vals{i, j}|\vals{<i})= f(\vals{\leq i}) \sum_{j=1}^m \frac{\nabla_\node p(\vals{i, j}|\vals{<i})}{p(\vals{i, j}|\vals{<i})}$ where we use that each discrete distribution is independent.
By unbiasedness of the GO gradient, the rest of the estimator is 0 in expectation, as we will show. 

Define $f_{j, k}=f(\vals{\leq i \setminus \vals{i, j}}, \vals{i, j}=k)$, $p_{j, k} = p_j(k|\vals{<i})$ and $P_{j, k} = \sum_{k'=1}^k p_j(k'|\vals{<i})$.
% Then:
% \begin{align}
%         &\sum_{j=1}^m (f_{j, \vals{i, j}} - f_{j, \vals{i, j} + 1}) \frac{\nabla_\node P_{j, \vals{i, j}} }{p_{j, \vals{i, j}}}\\
%         =&\sum_{j=1}^m f_{j, \vals{i, j}} \nabla_\node \log p_{j, \vals{i, j} } +  f_{j, \vals{i, j}}\frac{\nabla_\node P_{j, \vals{i, j} - 1} }{p_{j, \vals{i, j}}} - f_{j, \vals{i, j} + 1} \frac{\nabla_\node P_{j, \vals{i, j}} }{p_{j, \vals{i, j}}}
% \end{align}
Then the GO \tadditive{} is:
\begin{align}
    \additive{i}(\vals{\leq i}) = \sum_{j=1}^m I[\vals{i, j}<d]\Big( \bot\big(\frac{f_{j, \vals{i, j}}  - f_{j, \vals{i, j} + 1}}{p_{j, \vals{i, j}}}\big) (\magic(P_{j, \vals{i, j}}) - 1) \Big) 
    - \bot(f_{j, d})(\magic(\log p_{j, d}) - 1)
\end{align}
The first line will evaluate to the GO gradient estimator when differentiated, and the second to the single-sample score function gradient estimator. 

Note that this gives a general formula for implementing any unbiased estimator into \emph{Storchastic}: Use it as a control variate with the  score function subtracted to ensure interoperability with other estimators in the stochastic computation graph. 

\subsection{SPSA}
Simultaneous perturbation stochastic approximation (SPSA) \cite{spallMultivariateStochasticApproximation1992} is a gradient estimation method based on finite difference estimation. It stochastically perturbs parameters and uses two functional evaluations to estimate the (possibly stochastic) gradient. 
Let $\btheta$ be the $d$-dimensional parameters of the distribution $p_\btheta(\vals{i}|\vals{<i})$. 
SPSA samples $d$ times from the Rademacher distribution (a Bernoulli distribution with 0.5 probability for 1 and 0.5 probability for -1) to get a noise vector $\beps$. 
We then get two new distributions: $\vals{i, 1} \sim p_{\btheta + c\beps}$ and $\vals{i, 2} \sim p_{\btheta - c\beps}$ where $c>0$ is the perturbation size. 
The difference $\frac{f(\vals{i, 1}) - f(\vals{i, 2})}{2c \beps}$ is then an estimate of the first-order gradient. 
Higher-order derivative estimation is also possible, but left for future work.

An easy way to implement SPSA in \emph{Storchastic} is by using importance sampling (Appendix \ref{sec:importance-sampling}). 
Assuming $p_{\btheta + c\beps}$ and $p_{\btheta - c\beps}$ have the same support as $p$,  we can set the weighting function to $\bot\left(\frac{p(\vals{i, 1}|\vals{<i})}{p_{\btheta + c\beps}(\vals{i, 1}|\vals{<i})}\right)$ for the first sample, and $\bot\left(\frac{p(\vals{i, 2}|\vals{<i})}{p_{\btheta - c\beps}(\vals{i, 2}|\vals{<i})}\right)$ for the second sample.

To ensure the gradients distribute over the parameters, we define the \tmultipl{} as $\btheta\bot\left(\frac{p_{\btheta + c\beps}(\vals{i, 1}|\vals{<i})}{2c\beps p(\vals{i, 1}|\vals{<i})}\right)$ for the first sample and $-\btheta\bot\left(\frac{p_{\btheta + c\beps}(\vals{i, 2}|\vals{<i})}{2c\beps p(\vals{i, 2}|\vals{<i})}\right)$ for the second sample. 
This cancels out the weighting function, resulting in the SPSA estimator.

\subsection{Measure Valued Derivatives}
\emph{Storchastic} allows for implementing Measure Valued Derivatives (MVD) \cite{heidergottMeasurevaluedDifferentiationMarkov2008}, however, it is only unbiased for first-order differentiation and cannot easily be extended to higher-order differentiation. 
The implementation is similar to SPSA, but with some nuances.
We will give a simple overview for how to implement this method in \emph{Storchastic}, and leave multivariate distributions and higher-order differentiation to future work.

First, define the weak derivative for parameter $\theta$ of $p$ as the triple $(c_\theta, p^+, p^-)$ by decomposing $p(\vals{i}|\vals{<i})$ into the positive and negative parts $p^+(\vals{i}^+)$ and $p^- (\vals{i}^-)$, and let $c_\theta$ be a constant. 
For examples on how to perform this decomposition, see for example \cite{mohamedMonteCarloGradient2020}. 
To implement MVDs in \emph{Storchastic}, we use the samples from $p^+$ and $p^-$, and, similar to SPSA, treat them as importance samples (Appendix \ref{sec:importance-sampling}) for the zeroth order evaluation. 

That is, the \tproposal{} is defined over tuples $\sampleset_{i}=(\vals{i}^+, \vals{i}^-)$ such that $\fproposalcond{i} = p^+(\vals{i}^+)p^-(\vals{i}^-)$. The weighting function can be derived depending on the support of the positive and negative parts of the weak derivative. For weak derivatives for which the positive and negative part both cover an equal proportion of the distribution $p(\vals{i}|\vals{<i})$, the weighting function can be found using importance sampling by $\bot\left(\frac{p(\vals{i}^+|\vals{<i})}{2p^+(\vals{i}^+)}\right)$ for samples from the positive part, and $\bot\left(\frac{p(\vals{i}^-|\vals{<i})}{2p^-(\vals{i}^-)}\right)$ for samples from the negative part. This gives unbiased zeroth order estimation by using importance sampling. 

We then set $\fadditive{i}=0$ and use the following \tmultipl{}: $\fmultipl{i} = \theta \cdot \bot(c_\theta \frac{2p^+(\vals{i}^+)}{p(\vals{i}^+|\vals{<i})})$ for positive samples and $\fmultipl{i} = -\theta \cdot \bot(c_\theta \frac{2p^-(\vals{i}^-)}{p(\vals{i}^-|\vals{<i})})$ for negative samples. This will compensate for the weighting function by ensuring the importance weights are not applied over the gradient estimates. For the first-order gradient, this results in the MVD $\nabla_\node \theta \bot(c_\theta) (f(\vals{i}^+) - f(\vals{i}^-))$.

For other distributions for which $p^+$ and $p^-$ do not cover an equal proportion of $p$, more specific implementations have to be derived. 
For example, for the Poisson distribution one can implement its MVD by noting that $p^+$ has the same support as $p$. 
Then, we can use one sample from $p^+$ using the importance sampling estimator using score function (Appendix \ref{sec:importance-sampling}), and use a trick similar to the GO gradient by defining a \tadditive{} that subtracts the score function and adds the MVD, which is allowed since the MVD and score function are both unbiased estimators.

\section{Discrete VAE Case Study Experiments}
\label{sec:experiments}
\begin{table}[]
    \begin{tabular}{l|ll|ll}
                    & \multicolumn{2}{l|}{$2^{20}$ Bernoulli VAE} & \multicolumn{2}{l}{$10^{20}$ Discrete VAE} \\
                    & Train ELBO         & Validation ELBO        & Train ELBO        & Validation ELBO        \\ \hline
    Score@1         & 191.3              & 191.9                  & 206.3             & 206.7                  \\
    ScoreLOO@5 \cite{koolBuyREINFORCESamples2019}     & 110.8              & 110.4                  & 111.2             & 110.4                  \\
    REBAR@1 \cite{tuckerREBARLowvarianceUnbiased2017}        & 220.0              & 1000                   & 155.6             & 154.9                  \\
    RELAX@1 \cite{grathwohlBackpropagationVoidOptimizing2018}        & 210.6              & 205.9                  & 202.5             & 201.7                  \\
    Unordered set@5 \cite{koolEstimatingGradientsDiscrete2020} & 117.1              & 138.4                  & 115.4             & 117.2                  \\
    Gumbel@1 \cite{jangCategoricalReparameterizationGumbelsoftmax2017,maddisonConcreteDistributionContinuous2017}       & 107.0              & 106.6                  & 92.9              & 92.6                   \\
    GumbelST@1 \cite{jangCategoricalReparameterizationGumbelsoftmax2017}     & 113.0              & 112.9                  & 98.3              & 98.0                   \\
    ARM@1 \cite{yinARMAugmentREINFORCEMergeGradient2019}         & 131.3              & 130.8                  &                   &                        \\
    DisARM@1\cite{dongDisARMAntitheticGradient2020}         & 125.1              & 124.3                  &                   &                           
    \end{tabular}
    \caption{Test runs on MNIST VAE generative modeling. We report the lowest train and validation ELBO over 100 epochs. The number after the `@' symbol denotes the amount of samples used to compute the estimator. We note that the ARM and DiSARM methods are specific for binary random variables, and do not evaluate it in the $10^{20}$ discrete VAE.}
    \label{tbl:vae}
\end{table}
We report test runs on MNIST \cite{lecunMNISTHandwrittenDigit2010} generative modeling using discrete VAEs in Table \ref{tbl:vae}. We use Storchastic to run 100 epochs on both a latent space of 20 Bernoulli random variables and 20 Categorical random variables of 10 dimensions, and report training and test ELBOs.
We run these on the gradient estimation methods currently implemented in the PyTorch library.

Although results reported are worse than similar previous experiments, we note that we only run 100 epochs (900 epochs in \cite{koolEstimatingGradientsDiscrete2020}) and we do not tune the methods. 
However, the results reflect the order expected from \cite{koolEstimatingGradientsDiscrete2020}, where score function with leave-one-out baseline also performed best, closely followed by the Unordered set estimator.
Furthermore, the Gumbel softmax \cite{jangCategoricalReparameterizationGumbelsoftmax2017,maddisonConcreteDistributionContinuous2017} still outperforms the other score-function based estimators, although the results in \cite{koolEstimatingGradientsDiscrete2020} suggest that with more epochs and better tuning, better ELBO than reported here can be achieved.

These results are purely presented as a demonstration of the flexbility of the Storchastic library: Only a single line of code is changed to be able to compare the different estimators!
A more thorough and fair comparison, also in different settings, is left for future work. 

\end{document}